\def\year{2022}\relax
\documentclass[letterpaper]{article} 
\usepackage{aaai22}  
\usepackage{times}  
\usepackage{helvet}  
\usepackage{courier}  
\usepackage[hyphens]{url}  
\usepackage{graphicx} 
\usepackage{natbib}  
\usepackage{caption} 
\DeclareCaptionStyle{ruled}{labelfont=normalfont,labelsep=colon,strut=off} 
\usepackage{algorithm}
\usepackage{algorithmic}

\usepackage{commath}
\usepackage{amsmath}
\usepackage{amssymb}
\usepackage{amsthm}
\usepackage{bm}
\usepackage{color}
\usepackage{enumitem}
\usepackage{subfigure}

\newcommand{\E}{\mathbb{E}}

\newtheorem{assumption}{Assumption}
\newtheorem{theorem}{Theorem}
\newtheorem{lemma}{Lemma}
\newtheorem{Def}{Definition}
\newtheorem{corollary}[lemma]{Corollary}

\DeclareMathOperator*{\argmax}{argmax}

\usepackage{newfloat}
\usepackage{listings}
\floatstyle{ruled}
\newfloat{listing}{tb}{lst}{}
\floatname{listing}{Listing}
\title{Policy Learning for Robust  Markov Decision Process \\
with a Mismatched Generative Model}
\author {
    Jialian Li,\textsuperscript{\rm 1}
    Tongzheng Ren, \textsuperscript{\rm 2}
    Dong Yan, \textsuperscript{\rm 1}
    Hang Su, \textsuperscript{\rm 1}
    Jun Zhu, \textsuperscript{\rm 1}
}
\affiliations {
    \textsuperscript{\rm 1} Department of Computer Science and Technology, Beijing National Research Center for Information Science and Technology, Tsinghua-Bosch Joint Center for Machine Learning, Institute for Artificial Intelligence, Tsinghua University, Beijing 100084, China\\
    \textsuperscript{\rm 2} Department of Computer Science, UT Austin\\
    \texttt{lijialian7@163.com, tongzheng@utexas.edu, sproblvem@gmail.com, suhangss@mail.tsinghua.edu.cn, dcszj@mail.tsinghua.edu.cn}
}

\usepackage{bibentry}

\begin{document}

\maketitle

\begin{abstract}
In high-stake scenarios like medical treatment and auto-piloting, it's risky or even infeasible to collect online experimental data to train the agent. Simulation-based training can alleviate this issue, but may suffer from its inherent mismatches from the simulator and real environment. It is therefore imperative to utilize the simulator to learn a robust policy for the real-world deployment.
In this work, we consider policy learning for Robust Markov Decision Processes (RMDP), where the agent tries to seek a robust policy with respect to unexpected perturbations on the environments. Specifically, we focus on the setting where the training environment can be characterized as a generative model and a constrained perturbation can be added to the model during testing. Our goal is to identify a near-optimal robust policy for the perturbed testing environment, which introduces additional technical difficulties as we need to simultaneously estimate the training environment uncertainty from samples and find the worst-case perturbation for testing. To solve this issue, we propose a generic method which formalizes the perturbation as an opponent to obtain a two-player zero-sum game, and further show that the Nash Equilibrium corresponds to the robust policy. We prove that, with a polynomial number of samples from the generative model, our algorithm can find a near-optimal robust policy with a high probability. Our method is able to deal with general perturbations under some mild assumptions and can also be extended to more complex problems like robust partial observable Markov decision process, thanks to the game-theoretical formulation.
\end{abstract}

\section{Introduction}
\noindent Reinforcement Learning (RL) \cite{sutton1998introduction} aims to identify good policies that can solve the sequential decision-making problem, and has achieved amounts of incredible results on different challenging tasks \cite{mnih2013playing,mnih2015human,silver2017mastering}. Most of the current work focuses on the case where we can evaluate the performance of learned policies on the training environments, which is reasonable in multiple cases. For example, we can evaluate the policy for video games under the same environment since the environmental dynamics or game rules are the same for both training and testing. 

However, most of these algorithms need to interact with the environment and collect online experimental data to train the agent, which restricts the application of these algorithms in high-stake scenarios like medical treatment and auto-piloting. In certain cases, we have access to an additional simulator (e.g. for auto-piloting we can train the agent in a simulator that simulates the real traffic scenarios) which allows us to train the agent on the simulator. However, the simulator can substantially suffer from the domain mismatch from the real application scenarios, which may lead to a significant performance degeneration when deploying the model in real world.
We can instead utilize this imperfect simulator for training and find a relatively robust policy such that it can also work well in the real world. This indeed falls into the distributional robust optimization (DRO) problem \cite{2019Distributionally}.  

In this work, we consider DRO problems where the environment can be formalized as a Markov Decision Process (MDP) with finite steps. Several existing works~\cite{nilim2004robustness,iyengar2005robust} treat the difference of the training and testing MDP as a perturbation and the goal is to find a robust policy that ensures a near-optimal reward against the worst-case perturbation, which is so-called a Robust MDP (RMDP) problem. Many previous works~\cite[e.g.][]{nilim2004robustness,iyengar2005robust} assume that all the parameters of the training MDP and the perturbation set are known. Then the robust policy can be solved directly by (approximated) dynamic programming. However, in the practical problems, the simulator may be constructed as a complex system with physical computation \cite{2014Hydrology,2020Digital} which can only provide samples through interactions. Other works~\cite{2010percentile,2013Robust,ghavamzadeh2016safe} attempt to address problems under parameter uncertainty, but they generally neither take extra perturbation into consideration, nor provide sample complexity analysis for a near-optimal result. 

Solving RMDPs in an online manner is quite hard, because we might not be able to reach some potentially important states during online training. This might lead to policy's bad performance during testing. \citet{jiang2018pac} gives an extreme case where there is only one single state-action pair difference between the simulator and real environment, which makes the learned policy information-theoretically useless. To the best of our knowledge, there are no satisfactory results for solving RMDP in an online fashion. \citet{jiang2018pac} assumes that it has access to the testing environment and can get data to correct training simulator. However, we want to remark that such cases may not be applicable in numerous practical scenarios, as we cannot take such risk to collect online data in the testing environment. \citet{roy2017reinforcement} use model-free methods to solve the RMDP problem by assuming the perturbation set fixed and given. However, they do not provide the non-asymptotic sample complexity. In some bad cases of online RMDP, where some potential important states are hard to reach, the algorithm may suffer a large sample complexity. 

In this work, we consider the online training RMDP problems in a relaxed way. We assume the training environment is characterized via a generative model and the agent queries the generative model to gain samples from the training MDP. Although solving an MDP with a generative model has gained much interest \cite{kearns1999finite,azar2013minimax,cui2020plug}, the corresponding analysis in RMDP remains mostly open problem. We extend the fix perturbation set assumption in \cite{roy2017reinforcement} to general constraints that the perturbation may be dependent on the environment parameters. This would lead to a new challenge that our estimation error for the training environment can further influence the robust policy solving. 


\textbf{Our Contributions.} To address the aforementioned issues, we propose a method for solving RMDPs under general constraints. We first re-formulate the RMDP problem as a two-player zero-sum game by considering the perturbation as an adversarial player, and show that the Nash Equilibrium (NE) of the game corresponds to the robust policy and worst-case perturbation. Then we propose a new algorithm that uses a plug-in NE solver to find the robust policy. 

We provide rigorous non-asymptotic sample complexity bounds for our algorithm in certain cases. For RMDP with $S$ states, $A$ actions, and horizon $H$, we give non-asymptotic analysis for constraints with the Lipschitz condition. We show that with samples of order $\widetilde{O}((1+\lambda)^2S^2AH^4/\epsilon^2)$\footnote{The notation $\widetilde{O}$ represents the order ignoring logarithm term.}, our learned policy has an error no larger than $\epsilon$ with a high probability, where $\lambda$ is the Lipschitz constant. 

We finally remark that, since we use a game-theoretical framework to consider the environment uncertainty and robust policy solving separately, we can extend our method to problems with other uncertainty. We give an example on solving the robust partial-observable MDP problems.

\section{Problem Formulation}

In this section, we introduce the background of Markov Decision Process (MDP) and robust Markov Decision Process (RMDP) problem in this work.

\subsection{Markov Decision Process (MDP)}

We consider the finite-horizon Markov Decision Process (MDP) $M=\langle \mathcal{S},\mathcal{A},P,r,H\rangle$ where $\mathcal{S}$ is the state space and $\mathcal{A}$ is the action space. We denote $S=|\mathcal{S}|$ and $A=|\mathcal{A}|$. Transition probability $P(s,a)$ represents the probability distribution of transiting to states in $\mathcal{S}$ if taking action $a$ at state $s$. Reward function $r$ maps a state-action pair $(s,a)$ to a reward value $r(s,a)\in[0,1]$. One episode of interaction has a depth of $H$ which is denoted as $[H]=\{1,2,...,H\}$.

Usually for finite-horizon MDPs, the optimal actions for one state at different depths can be different. For clarity, we assume the state sets for different depths are disjoint. Formally, we use $\mathcal{S}_h$ to denote the set of states with depth $h$ and $D=\max_{h}|\mathcal{S}_h|$, where $\mathcal{S}=\cup_{h=1}^H \mathcal{S}_h$ and $\mathcal{S}_h\cap \mathcal{S}_{h'}=\emptyset$ if $h\neq h'$. Furthermore we denote the simplex of all possible transition probability vectors on $P(s,a)$ as $\Delta_{sa}$. 

For clarity, we assume an initial state $s_0$ at depth $1$. At state $s\in\mathcal{S}_h$, the agent chooses one action $a\in\mathcal{A}$. Then the environment turns to state $s'\in\mathcal{S}_{h+1}$ with a probability of $P(s'|s,a)$. After $H$ times of interactions from depth $1$, one episode ends. For the convenience of notation, we define a terminal state $s_\mathcal{T}$ at depth $H+1$ and the environment evolves to $s_{\mathcal{T}}$ from depth $H$.

A policy $\pi$ of the agent maps each state $s\in\mathcal{S}$ to an action in $\mathcal{A}$. Usually, each state's policy can be a distribution over $\mathcal{A}$. For a finite horizon MDP, there always exists a deterministic optimal policy\cite{puterman2014markov}. Here without further clarification, our ``policy'' indicates a deterministic policy for the agent.

Consider state $s\in \mathcal{S}_h$, we use $V^\pi_{M}(s)$ to represent the expected reward, following the policy $\pi$ in MDP $M$ with transition $P$:
\begin{small}
$$V^\pi_{M}(s)=\E_{P}\left[\sum_{h'=h}^Hr(s_{h'},\pi(s_{h'}))|s_h=s\right],$$
\end{small}
where $\mathbb{E}_P$ indicates the expectation over transition function.
Similarly, we define the $Q$ values for state-action pairs as
\begin{small}
$$Q^\pi_{M}(s,a)=\E_P\left[r(s,a) + \sum_{h'=h+1}^Hr(s_{h'},\pi(s_{h'}))\right].$$
\end{small}
The Bellman equation for MDP $M=\langle \mathcal{S},\mathcal{A},P,r,H\rangle$ at state $s\in\mathcal{S}_h$ is
\begin{small}
$$V^\pi_{M}(s) = r(s,\pi(s))+\sum_{s'\in\mathcal{S}_{h+1}}P(s'|s,a) V^\pi_{M}(s').$$
\end{small}
Specifically, we define $V^\pi_{M}(s_\mathcal{T})=0$ for any $\pi$ or $M$.
Our final goal is to find a policy $\pi'$ such that
$$\pi'=\argmax_\pi V^\pi_{M}(s_0).$$

\subsection{Robust Markov Decision Process (RMDP)}

We consider the problem where the environments at training and testing phases can be different. Denote the testing MDP as $M^*=\langle \mathcal{S},\mathcal{A},P^*,r,H\rangle$ and the training MDP as ${M}^s = \langle \mathcal{S},\mathcal{A},P^s,r,H\rangle$. For clarity, here we only assume the transitions are different between $M^*$ and $M^s$, while our techniques to handle the transition can be easily generalized to handle the rewards. 

During the training time, $P^*$ is not available for the agent and only samples from $P^s$ are accessible via a generative model. At each time step, the agent sends a state-action pair $(s,a)$ to the generative model and gains $(s',r(s,a))$, where $s'$ is sampled from $P^s(s,a)$. To enable a policy to transfer from the training environment to the testing environment, 
$P^s$ and $P^*$ should be to some degree similar. Here we use a constraint to describe their connection. We use $\prod$ to denote the Cartesian product and define $\mathcal{P}=\prod_{(s,a)}\Delta_{sa}$ to be the set composed of all the transition probabilities. Here denote dimension of $\mathcal{P}$ to be $G\leq SAD$. Then we assume a known constraint function $\mathcal{C}$ which maps each $P\in\mathcal{P}$ to a set $\mathcal{C}(P)\subseteq \mathcal{P}$, and also assume that the testing transition function $P^*$ locates in $\mathcal{C}(P^s)$. Further we subtract the transition $P$ from $\mathcal{C}(P)$ to define the perturbation set as
\begin{equation}
U(P):=\{p\in [-1,1]^{G}:P+p\in\mathcal{C}(P)\}.
\label{eq:perturbation}
\end{equation}
If $U(P^s)$ contains only elements near $0$ (i.e., $\|p\|$ is small), then $P^*$ is constrained to be within a neighborhood of $P^s$. Furthermore, we use $\mathcal{C}(M)$ to denote the set of MDPs whose transition functions are located in $\mathcal{C}(P)$.

In order to learn a robust policy, we need to consider the worst-case environment for each policy. We define a worst-case $V$ value for policy $\pi$ as
\begin{small}
\begin{equation}
\widetilde{V}^\pi(s)=\min_{M\in\mathcal{C}(M^s)}V^\pi_{M}(s),
\end{equation}
\end{small}
and the optimal robust policy $\pi^*$ can be defined as
$$\pi^*=\argmax_{\pi}\widetilde{V}^\pi(s_0).$$
For a policy $\pi$, we define the error between $\pi$ and $\pi^*$ as
\begin{small}
\begin{equation}
    \label{eq:pi_error}
    Err(\pi):=\widetilde{V}^{\pi^*}(s_0)-\widetilde{V}^{{\pi}}(s_0).
\end{equation}
\end{small}
Our learning goal is to find a policy $\pi'$ such that
\begin{small}
\begin{equation}
\label{eq:goal}
\mathbb{P}(Err(\pi')\leq \epsilon)\geq 1-\delta,
\end{equation}
\end{small}
for some $\delta\in(0,1)$ and $\epsilon>0$. Here $\mathbb{P}(\cdot)$ denotes the probability for some event. A key for this learning target is that the optimal policy for $M^s$ might work poorly under $M^*$, which urges for a robust policy for testing phase.

\subsection{A Simple Case for Intuition}

Here we give an simple example to show that robustness is indeed an import issue under the model mismatch problems. 

\begin{figure}[hbt]
    \centering
    \includegraphics[width=.35\textwidth]{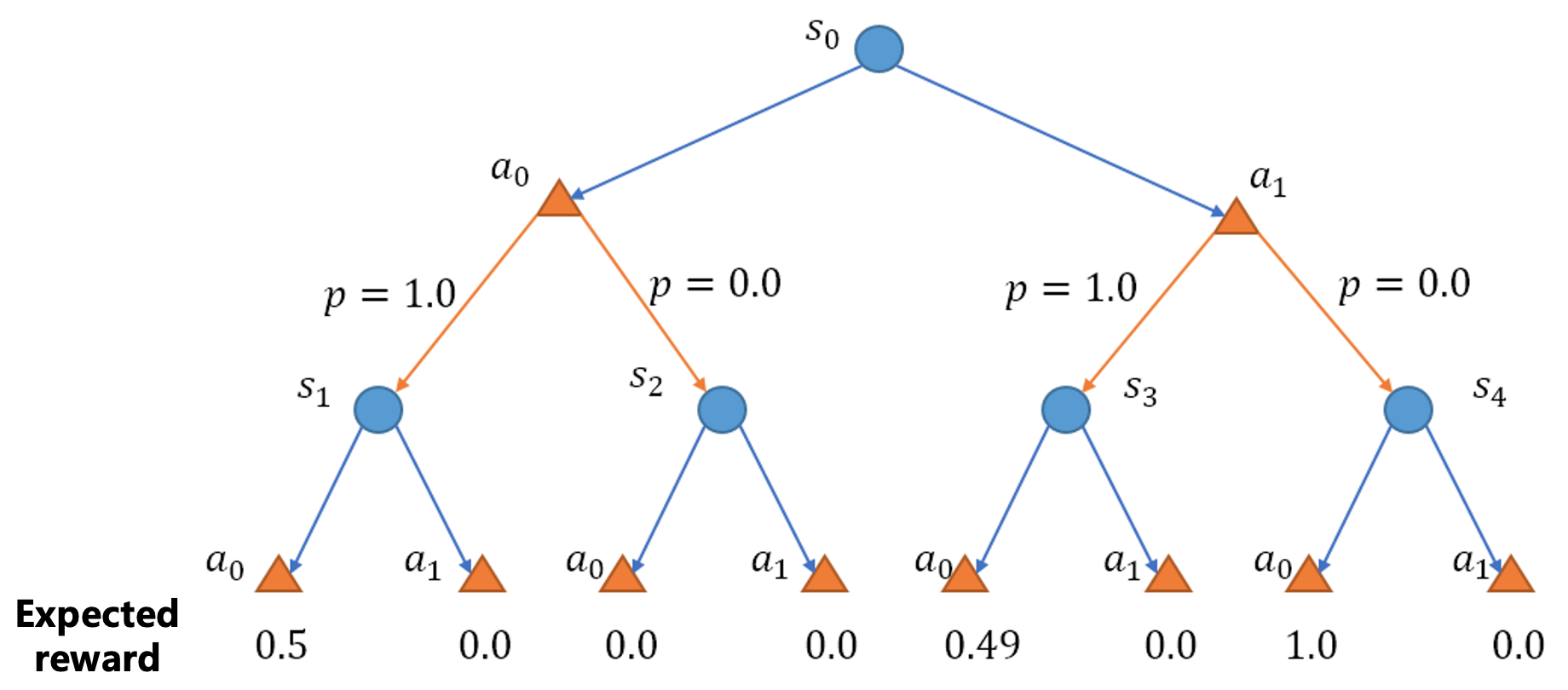}
    \caption{The structure of the MDP}
    \label{fig:mdp}
\end{figure}

We consider a training MDP $M^s$ with only 2 layers. Depth 1 has only one state, i.e. $s_0$ and depth 2 has four states. The agent has two actions $a_0$ and $a_1$ at each state. The agent receives a reward at the end of one episode. The structure of the MDP and the rewards are given in Fig.~\ref{fig:mdp}. It easy to see that the optimal strategy for $M^s$ is always to choose $a_0$ and the optimal reward is $0.5$. 

Assume that the true MDP $M^*$ and $M^s$ differs in the transitions. We assume a perturbation range $u$ for this example. That is, $P^*(s_2|s_0, a_0)\in[0,u)$  and $P^*(s_4|s_0,a_1)\in[0,u)$. Now if $u>0.02$, the robust policy for the agent is to choose $a_1$ at $s_0$. It can be seen that methods concentrating on $M^s$ can hardly identify this robustness issue since they do not pay attention to $s_4$ whose reaching probability in $M^s$ is $0$. 

We aim to use this simple case to show that: (1) the optimal policy of $M^s$ may be a bad policy for $M^*$; (2) the online training process for RMDP can be very inefficient. Thus we test 4 methods: (1) OPT: solving the optimal policy of $M^s$; (2) RPS: our robust policy solving method; (3) RQ-learning: the robust version of Q-learning; and (4) RSARSA: the robust version of SARSA. Here OPT and RPS are trained with the generative model and RQ-learning and RSARSA \cite{roy2017reinforcement} solves a one-step mini-max optimization. Each method can interact with the training environment for $20000$ times. We choose $u\in\{0.001,0.005,0.01,0.05,0.1,0.5\}$. We give detailed information about this experiment is given in the appendix.

\begin{figure}[hbt]
    \centering
    \subfigure[Worst perturbation]{\includegraphics[width=.23\textwidth]{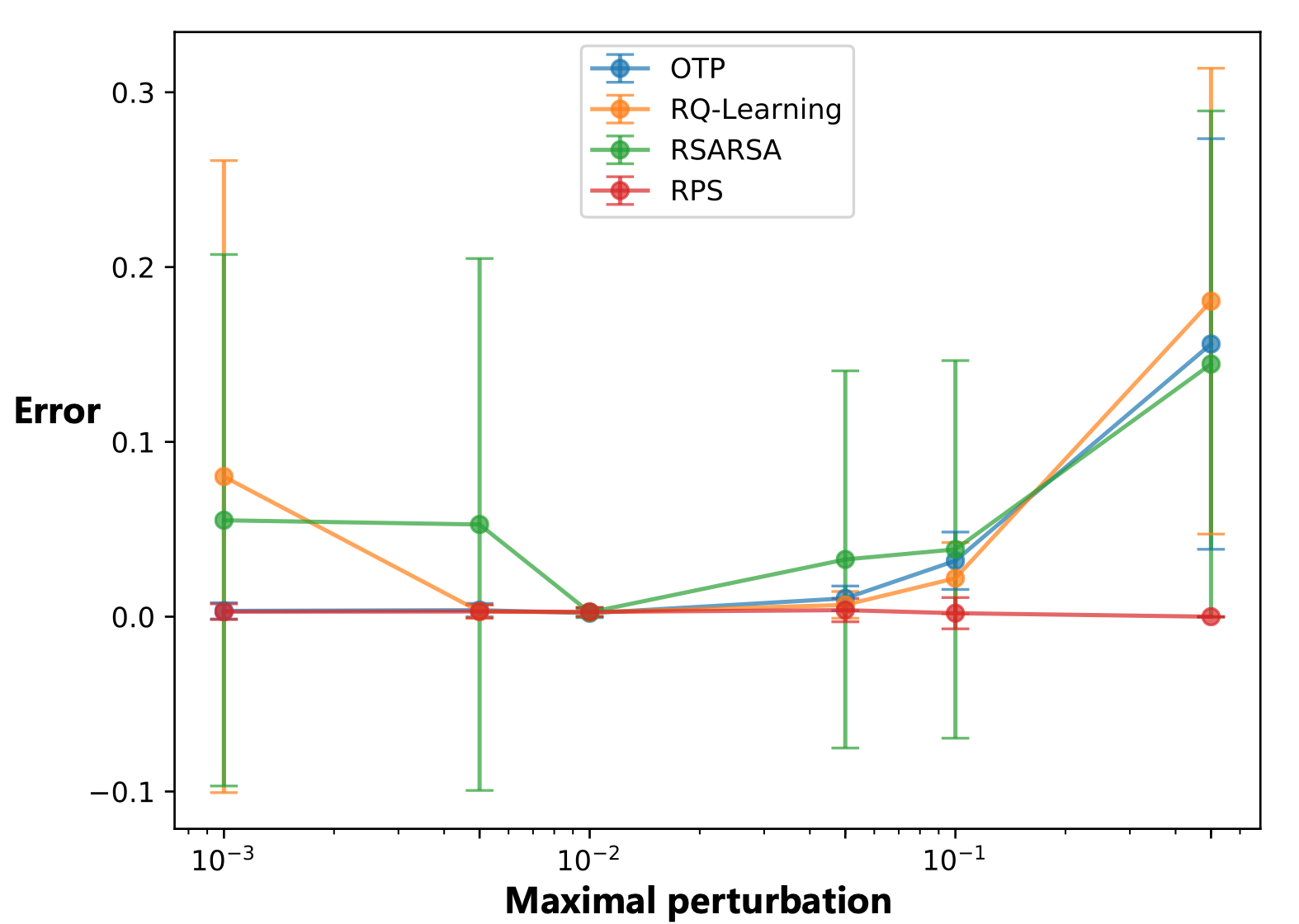}\label{fig:err}}
    \subfigure[Random perturbation]{\includegraphics[width=.23\textwidth]{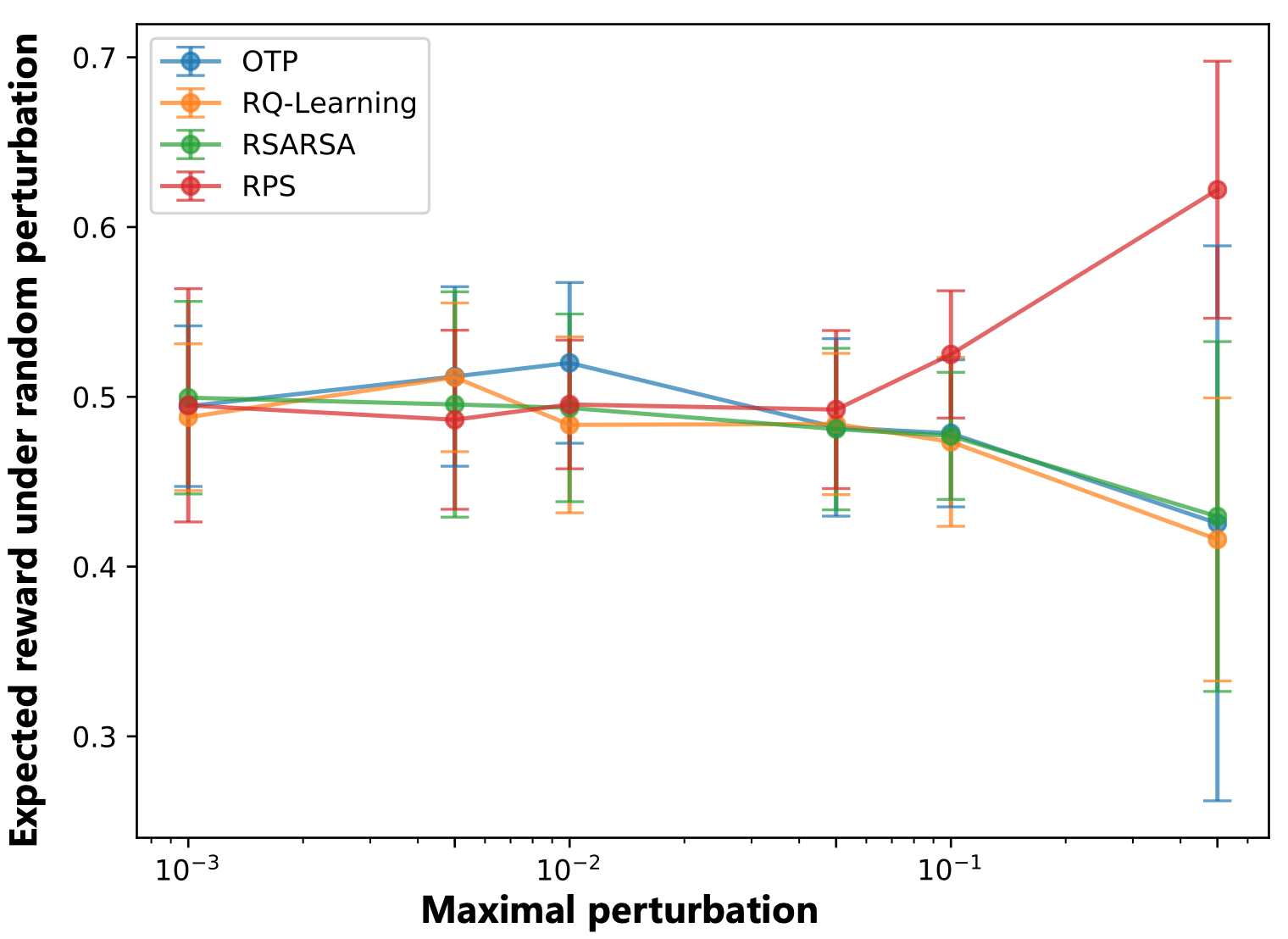}\label{fig:rand_res}}
    \caption{Results on 4 methods}
\end{figure}

The results under worst-case perturbation are shown in Fig.~\ref{fig:err}. It can be shown that only RPS can always finding the robust solution while other three methods suffer a large error for a large $u$. OTP suffer a large error because it ignores the effect of the perturbation. Online methods RSARSA and RQ-learning also fail because these online methods cannot reach $s_4$ at all. They cannot give a proper estimation for $s_4$ and cannot find the near-optimal robust solution.

We also show the results on random perturbation in Fig.~\ref{fig:rand_res}. The random perturbation here means that the difference between the training and testing transitions is sampled uniformly from the perturbation range. It can be seen that for a relatively large perturbation, a robust policy indeed gives higher expected rewards, since it takes the risks into consideration. 

\section{Methodology}
We now present a method to solve for the robust policy in RMDP. It is hard to solve the problem when the uncertainty estimation process and the robust policy solving process are entangled. We treat them separately by re-formalizing the RMDP as a two-player zero-sum game where the agent is a player maximizing reward and the perturbation is the other player minimizing the reward. We deal with the training estimation error by using sufficient samples to build an proper model for the game. Then we can use a plug-in Nash Equilibrium (NE) solver to solve the game. The NE corresponds to the robust policy and the worst-case perturbation.

The idea of solving a minimax dynamic programming have been used in previous works~\cite{nilim2004robustness,iyengar2005robust}. The key advantage of considering a two-player zero-sum game rather than simply a minimax dynamic programming is that we can handle the perturbation in a more general way. By doing so, we can estimate the environment uncertainty with RL methods and solve the robust policy with game-theoretical techniques. This separation can help us to handle problems with more complex environment uncertainty or mismatch constraints. We give more examples in the next section.

In this section, we first re-formulate the RMDP problem as a two-player zero-sum game and then present a method for RMDP with general constraint functions. Finally we give a non-asymptotic sample complexity result for our method under a mild assumption.


\subsection{RMDP as a Two-player Zero-sum Game}
\label{sec:tzg}

\begin{algorithm}[tb]
	\caption{Solving the robust policy with a Nash Equilibiurm (NE) solver}
	\label{alg:ubvp}
	\begin{algorithmic}[1]
		\STATE {\bfseries Input:} The generative model $M$, the perturbation constraint $U$ and the NE solver $\mathcal{O}$
		\FOR{$(s,a)$ in $\mathcal{S}\times\mathcal{A}$}
        \STATE Query $M$ with $(s,a)$ for $N$ times
        \STATE Calculate the transition empirical mean $\bar{P}(s,a)$ and the reward empirical mean $\bar{r}(s,a)$
		\ENDFOR
		\STATE Get an empirical model $\bar{M}=\langle \mathcal{S},\mathcal{A}, \bar{P}, \bar{r}, H\rangle$.
		\STATE Solve the NE $(\widehat{\pi}^*,\widehat{\sigma}^*)=\mathcal{O}(\bar{M},U(\bar{P}))$
		\STATE {\bfseries Output:} The robust policy $\widehat{\pi}^*$
	\end{algorithmic}
\end{algorithm}

We consider MDP $M$ with a given transition $P$ as well as the corresponding $U(P)$ defined in Eq.~\eqref{eq:perturbation}. We denote $\sigma$ as the elements in $U(P)$. For convenience, we use $M+\sigma$ to denote a perturbed MDP with transition function $P+\sigma$. The goal to learn a robust policy is to solve a maximin problem 
\begin{equation}
\label{eq:maximin_opt}
    \max\limits_\pi\min\limits_{\sigma\in U(P)} V^\pi_{M+\sigma}(s_0).
\end{equation}
With the above formulation, we can consider this task as a two-player zero-sum game where player 1 chooses $\pi$ to maximize $V^\pi_{M+\sigma}(s_0)$ while player 2 chooses $\sigma$ to minimize $V^\pi_{M+\sigma}(s_0)$. From this point of view, we define $\mathcal{V}$ and $\mathcal{Q}$ values to replace the $V$ and $Q$ values as 
\begin{align}
\mathcal{V}_{M}^{\pi,\sigma}(s):=V_{M+\sigma}^\pi(s),\
\mathcal{Q}_{M}^{\pi,\sigma}(s,a):=Q_{M+\sigma}^\pi(s,a).
\label{eq:game_q}
\end{align}

Here $(\pi,\sigma)$ are the policy pair of the two players and $\mathcal{V}_{M}^{\pi,\sigma}(s)$ represents the expected value for player 1 at state $s$ of depth $h$. The expected reward for player 2 is accordingly $-\mathcal{V}_{M}^{\pi,\sigma}(s)$.

We further use $\sigma(s,a)$ and $\sigma(s'|s,a)$ to denote the perturbations added to $P(s,a)$ and $P(s'|s,a)$, respectively. Then the Bellman Equation for this game can be shown as:
\begin{small}
\begin{align*}
\mathcal{V}_{M}^{\pi,\sigma}(s)
&= 
r(s,\pi(s)) \\
&+ \sum_{s'}\Big(P(s'|s,\pi(s))+\sigma(s'|s,\pi(s))\Big)\mathcal{V}_{M}^{\pi,\sigma}(s').
\end{align*}
\end{small}

Solving problem~\eqref{eq:maximin_opt} is equivalent to solving the Nash Equilibrium (NE) of the game. More specifically, the NE solution is corresponding to the robust policy and the worst-case perturbation. Note that there might be multiple NEs for one game. In two-player zero-sum games, all the NEs have the same reward values. Thus any NE can be considered as a solution of the original problem.

Now we turn to solve the NE of the game. We assume that we have access to an NE solver $\mathcal{O}$ which maps the MDP $M$ and perturbation set $U(P)$ to a policy pair $(\pi',\sigma')=\mathcal{O}(M,U(P))$ that satisfies
\begin{align*}
\pi' =& \mathop{\arg\max\limits}_{\pi} \mathcal{V}_{M}^{\pi,\sigma'}(s_0),\
\sigma' =  \mathop{\arg\min}_{\sigma\in U(P)} \mathcal{V}_{M}^{\pi',\sigma}(s_0).
\end{align*}

Then the policy pair $(\pi^*,\sigma^*)=\mathcal{O}(M^s,U(P^s))$ is exactly an NE for our learning goal of problem \eqref{eq:maximin_opt} and $\pi^*$ is exactly the robust policy we aim to solve. 


\subsection{Method for General Constraints}


Based on the above formulation, our method is simple and intuitive. We consider our setting that a generative model $M^s=\langle \mathcal{S},\mathcal{A}, P^s, r, H\rangle$ is given for training. For each state-action pair $(s,a)$, we query the generative model for $N$ times and gain corresponding samples for next steps and rewards, denoted as $\{s_t, r_t\}_{t=1}^N$. The number $N$ is calculated based on the target $\epsilon$ and $\delta$, as we will give in Sec.\ref{sec:main_results}. Since we assume the reward to be deterministic and we have visited each state-action pair for $N$ times, we already know the reward $r$. Here use $\mathbb{I}[\cdot]$ as the indicator function. With these $N$ samples, we estimate the transition function at $(s,a)$ with
\begin{small}
\begin{align}
\label{eq:em_r}
\widehat{P}(s'|s,a) =&  \frac{1}{N}\sum\limits_{t=1}^{N} \mathbb{I}[s_t=s'].
\end{align}
\end{small}

We construct an empirical RMDP $\widehat{M}=\langle \mathcal{S},\mathcal{A}, \widehat{P},r, H\rangle$ and calculate its perturbation set $U(\widehat{P})$. We directly construct a two-player zero-sum game with $\widehat{M}$ and $U(\widehat{P})$, as in Sec.~\ref{sec:tzg}. We apply the NE solver $\mathcal{O}$ to this game to output the policy pair $(\widehat{\pi}^*,\widehat{\sigma}^*)=\mathcal{O}(\widehat{M},U(\widehat{P}))$ such that:
\begin{small}
\begin{align}
\widehat{\pi}^* = \mathop{\arg\max\limits}_{\pi} {\mathcal{V}}_{\widehat{M}}^{\pi,\widehat{\sigma}^*}(s_0), \
\widehat{\sigma}^* = \mathop{\arg\min}_{\sigma\in U(\widehat{P})} {\mathcal{V}}_{\widehat{M}}^{\widehat{\pi}^*,\sigma}(s_0).
\label{eq:sol_delta}
\end{align}
\end{small}
For convenience, here we denote $\widehat{\mathcal{V}}^*(s)={\mathcal{V}}^{\widehat{\pi}^*,\widehat{\sigma}^*}_{\widehat{M}}(s)$, $\widehat{\mathcal{Q}}^*(s,a)={\mathcal{Q}}^{\widehat{\pi}^*,\widehat{\sigma}^*}_{\widehat{M}}(s,a)$ for $s\in\mathcal{S}_h$. 

Notice that $(\widehat{\pi}^*,\widehat{\sigma}^*)$ is solved in the constraint $\mathcal{C}(\widehat{M})$, rather than $\mathcal{C}({M}^s)$. In the next section, we will show that under some mild assumptions on $\mathcal{C}$, with a polynomial number of samples, $\widehat{M}$ can be a good approximation for $M^s$ and our solution $\widehat{\pi}^*$ is guaranteed to be a near-optimal policy.

\subsection{Theoretical Analysis}
\label{sec:main_results}
Our method above uses a plug-in NE solver to solve the empirical RMDP. Besides the estimation error caused by the sampling, the policy $\widehat{\pi}^*$ is calculated from the perturbation $U(\widehat{P})$ rather than $U(P^s)$, which causes another error. Intuitively, for a large $N$, $\widehat{M}$ would be sufficiently close to $M^s$. Thus the gaps between $\widehat{P}$ and $P^s$ would be bounded. If the error caused by the difference between $U(\widehat{P})$ and $U(P^s)$ can also be bounded, we can show that $\widehat{\pi}^*$ is a near-optimal policy. Here we make mild assumptions for $\mathcal{C}$ under which we can gain a near-optimal $\widehat{\pi}^*$ from a polynomial number of samples from the generative model. 

To show that the policy $\widehat{\pi}^*$ learned by the plug-in NE solver is near optimal, we need to show that the error term $Err(\widehat{\pi}^*)$ defined in Eq.~\eqref{eq:pi_error} can be bounded by a small value with a high probability. In general, we cannot guarantee $\widehat{\pi}^*$ to be the near-optimal robust policy even when $N$ is large, because the difference between $U(P^s)$ and $U(\widehat{P})$ can cause a large error. We add a mild assumption for the constraint $\mathcal{C}$ and then we can give sample complexity analysis.

\begin{assumption}[Lipschitz condition for the perturbation]
    \label{assp:constraint}
    For two MDPs $M=\langle \mathcal{S}, \mathcal{A}, P, r, H\rangle$ and $M'=\langle \mathcal{S}, \mathcal{A}, P', r, H\rangle$, the constraint $\mathcal{C}$ and its corresponding perturbation $U$ satisfy that for any $\sigma\in U(P)$ and the projection of $\sigma$ onto $U(P^\prime)$
    $$\sigma':=\mathop{\arg\min}\limits_{\sigma''\in U(P')}\left(\max\limits_{(s,a)}||{\sigma}''(s,a) -\sigma(s,a)||_1\right),$$
    there exits a constant $\lambda$ such that for any $(s,a)$,
    \begin{small}
    \begin{equation}
       ||{\sigma}(s,a) -\sigma'(s,a)||_1\leq \lambda||P(s,a) -P'(s,a)||_1.
    \end{equation}
    \end{small}
\end{assumption}
This assumption indicates that the distance of $\sigma'\in U(P')$ and $\sigma\in U(P)$ satisfies a Lipschitz condition, where $\sigma'$ is the closest perturbation in $U(P')$ to $\sigma$.  
This is a relatively general assumption for constraints. For example, the total variance distance, where the total variance of the perturbation is bounded, satisfies this assumption by choosing $\lambda=2$. 

For constraints satisfying Assumption \ref{assp:pwc}, our learned policy is guaranteed to be near-optimal, as shown by the following theorem. 
\begin{theorem}
\label{thm:main_thm}
For a given $(\epsilon, \delta)$, where $\epsilon\in(0,H)$ and $\delta\in(0,1)$, and a constraint $\mathcal{C}$ satisfying Assumption~\ref{assp:constraint}, if 
\begin{equation}
    N\geq \frac{8(1+\lambda)^2H^4D\ln(2/\delta')}{\epsilon^2},
\end{equation}
then with a probability no less than $1-\delta$, we have
\begin{equation}
    Err(\widehat{\pi}^*)\leq\epsilon.
\end{equation}
\end{theorem}
The complete proof for this theorem is given in Appendix~\ref{app:main_thm}. Below we give a brief proof overview.
\begin{proof}
The proof process can be decomposed into 3 steps. In the first step we decompose $Err(\widehat{\pi}^*)$ into 2 terms, and then we provide upper bound for the two terms correspondingly. 

We first define four auxiliary policies as
\begin{small}
\begin{align}
\sigma^* = & \mathop{\arg\min}\limits_{\sigma\in U(P^s)} \mathcal{V}^{\pi^*,\sigma}_{M^s}(s_0),\ \sigma' = \mathop{\arg\min}\limits_{\sigma\in U(P^s)} \mathcal{V}^{\widehat{\pi}^*,\sigma}_{M^s}(s_0),\label{eq:sigma_prime}\\
\widetilde{\sigma}^* = & \mathop{\arg\min}\limits_{\sigma\in U(P^s)}(\max\limits_{(s,a)}||\widehat{\sigma}^* -\sigma||_1),\ 
\widetilde{\sigma}' = \mathop{\arg\min}\limits_{\sigma\in U(\widehat{P})}(\max\limits_{(s,a)}||\sigma' - \sigma||_1).\label{eq:sigma_tilde_prime}
\end{align}
\end{small}
The first two represent the best responses to $\pi^*$ and $\widehat{\pi}^*$ respectively under $U(P^s)$. $\widetilde{\sigma}^*$ is the projection of $\widehat{\sigma}^*$ (recall the definition in Eq.~\eqref{eq:sol_delta}, and notice that $\widehat{\sigma}^*\in U(\widehat{P})$.) onto $U(P^s)$. The last $\widetilde{\sigma}'$ is the projection of $\sigma'$ onto $U(\widehat{P})$.

Now we sketch the proof in 3 steps.

\textbf{Step 1. } 
We first provide an upper bound for $Err(\widehat{\pi}^*)$ with the following lemma.
\begin{lemma}
\label{lem:err_upper_bound}
\begin{equation}
Err(\widehat{\pi}^*)\leq e_1 + e_2,
\end{equation}	where
\begin{small}
$$e_1:=\mathcal{V}^{\pi^*,\widetilde{\sigma}^*}_{M^s}(s_0) - \mathcal{V}^{{\pi}^*,\widehat{\sigma}^*}_{\widehat{M}}(s_0),\  e_2:= \mathcal{V}^{\widehat{\pi}^*,\widetilde{\sigma}'}_{\widehat{M}}(s_0) -  \mathcal{V}^{\widehat{\pi}^*,{\sigma}'}_{M^s}(s_0).$$
\end{small}
\end{lemma}
Here $e_1$ measures the error from transferring $\widehat{\sigma}^*$ on $\widehat{M}$ to $\widetilde{\sigma}^*$ on $M^s$ given fixed $\pi^*$. Similarly, $e_2$ measures the error under $\hat{\pi}^*$.  We then bound the two errors separately.

\textbf{Step 2. }
For $e_1$, notice that the difference between the two $\mathcal{V}$ values is caused by the difference of transitions and perturbations. We can decompose $e_1$ into state-action pair-wise error terms with Lemma~\ref{lem:v_gap_decompose_gen} in Appendix~\ref{app:lemmas}. For clarity, we use $\mathcal{V}^{{\pi}^*,\widehat{\sigma}^*}_{\widehat{M},h}$ to denote the vector composed of all $\mathcal{V}^{{\pi}^*,\widehat{\sigma}^*}_{\widehat{M}}(s)$ where $s\in\mathcal{S}_{h}$. Then we have
\begin{small}
\begin{align}
& \mathcal{V}^{\pi^*,\widetilde{\sigma}^*}_{M^s}(s_0) - \mathcal{V}^{{\pi}^*,\widehat{\sigma}^*}_{\widehat{M}}(s_0) \notag\\
= & \sum_{h=1}^H\sum_{s\in \mathcal{S}_h}\sum_{a}{\xi}^{\pi^*,\widetilde{\sigma}^*}_{M^s}(s,a)\left( (P^s(s,a)-\widehat{P}(s,a))^\top \mathcal{V}^{{\pi}^*,\widehat{\sigma}^*}_{\widehat{M},h+1}\right) \label{eq:e1_p_gap}
\\
& + \sum_{h=1}^H\sum_{s\in \mathcal{S}_h}\sum_{a}{\xi}^{\pi^*,\widetilde{\sigma}^*}_{M^s}(s,a)\left( (\widetilde{\sigma}^*(s,a)-\widehat{\sigma}^*(s,a))^\top \mathcal{V}^{{\pi}^*,\widehat{\sigma}^*}_{\widehat{M},h+1}\right).\label{eq:e1_sigma_gap}
\end{align}
\end{small}
where $\xi^{\pi,\sigma}_M(s,a)$ is defined as the probability of reaching $(s,a)$ on MDP $M$ following policy $\pi$ and perturbation $\sigma$. 

Here the first summation term characterizes the estimation error caused by the samples, and the second characterizes the error caused by using an improper perturbation set.

Consider $(P^s(s,a)-\widehat{P}(s,a))^\top \mathcal{V}^{{\pi}^*,\widehat{\sigma}^*}_{\widehat{M},h+1}$ in Eq.~\eqref{eq:e1_p_gap}. For general constraint $\mathcal{C}$,  $\widehat{\sigma}^*$ can be correlated with $\widehat{P}$. Therefore, $\widehat{P}$ and $\mathcal{V}^{{\pi}^*,\widehat{\sigma}^*}_{\widehat{M},h+1}$ are not independent random variables. 
According to the Azuma's inequality on $\ell_1$-norm, we have that with a probability at least $1-\delta'$,
\begin{equation}
    ||P^s(s,a)-\widehat{P}(s,a)||_1\leq \alpha(N),
\end{equation}
where we define
	\begin{equation}
	\label{eq:alpha_def}
    \alpha(N)=\sqrt{2D\ln(2/\delta')/N}.
\end{equation}
Recall our definition of $D\geq|\mathcal{S}_{h}|$ for all $h\in[H]$. This property holds for all $(s,a)$ pairs.

Using the Cauchy-Schwarz inequality, we have that
\begin{small}
\begin{align}
    &(P^s(s,a)-\widehat{P}(s,a))^\top \mathcal{V}^{{\pi}^*,\widehat{\sigma}^*}_{\widehat{M},h+1} \nonumber
    \\
    \leq & ||P^s(s,a)-\widehat{P}(s,a)||_1||\mathcal{V}^{{\pi}^*,\widehat{\sigma}^*}_{\widehat{M},h+1}||_\infty
    \leq H\alpha(N).
\end{align}
\end{small}

For the $(\widetilde{\sigma}^*(s,a)-\widehat{\sigma}^*(s,a))^\top \mathcal{V}^{{\pi}^*,\widehat{\sigma}^*}_{\widehat{M},h+1}$ in Eq.~\eqref{eq:e1_sigma_gap}, we use the Cauchy–Schwarz inequality and Assumption~\ref{assp:constraint} to get
\begin{small}
\begin{align}
    & (\widetilde{\sigma}^*(s,a)-\widehat{\sigma}^*(s,a))^\top \mathcal{V}^{{\pi}^*,\widehat{\sigma}^*}_{\widehat{M},h+1} 
    \leq \lambda H\alpha(N).
\end{align}
\end{small}

Recall our definition for reaching probability $\xi^{\pi,\sigma}_M(s,a)$ and we have that $\sum_{s\in \mathcal{S}_h,a\in\mathcal{A}}{\xi}^{\pi,{\sigma}}_{M}(s,a)=H$.

Notice that $\alpha$ is a decreasing function with respect to $N$, and converges to 0 as $N$ goes to infinity. Then with 
\begin{small}
\begin{equation}
    N\geq \alpha^{-1}(\epsilon/(2(1+\lambda)H^2))=\frac{8(1+\lambda)^2H^4D\ln(2/\delta')}{\epsilon^2},
    \label{eq:general_N}
\end{equation}
\end{small}
we have that $e_1\leq \epsilon/2$. 

\textbf{Step 3.} Using similar analysis techniques (refer to Appendix~\ref{app:e2}) and $N$ in Eq.~\eqref{eq:general_N}, we also have $e_2\leq \epsilon/2$.
Finally, we choose $\delta'=\delta/(4SA)$ and use the union bound to get that with a probability no less than $1-\delta$, $Err(\widehat{\pi}^*)\leq \epsilon.$
\end{proof}
By roughly considering $D$ as $S$, this theorem shows that with a sample complexity of $\widetilde{O}(S^2AH^4(1+\lambda)^2/\epsilon^2)$, our method with a plug-in NE solver can find an $\epsilon$-optimal policy with a high probability. Further if our NE solver returns an approximate NE with a value error $\epsilon'$, then we can simply replace $\epsilon$ with $\epsilon-\epsilon'$ to calculate $N$.

\section{Cases and Extensions}
\label{sec:case_and_extension}

In this section, we give some examples for RMDPs. Finally we give a further extension which is also suitable to apply our game-theoretical framework.

\subsection{Pair-Wise Constraint (PWC)} 
\label{sec:pwc}

Here we consider a specific kind of constraints, where the constraint $\mathcal{C}$ has independent effect on each state-action pair. We call this as the \textit{Pair-Wise Constraint} (PWC). Many practical problems provide constraints directly on state-action pairs, such as Total Variation distance bound on transitions. These constraints can be included in the PWC.

For PWC problems, we can implement a simple NE solver based on value back-propagation, which has a quite similar form to methods in previous works \cite{nilim2004robustness,iyengar2005robust}. Detailed definition for PWC and its NE solver are given in Appendix~\ref{sec:pwc_results}. Here we only give its theoretical result.
\begin{theorem}
\label{thm:pwc_thm}
For a given $(\epsilon, \delta)$, where $\epsilon\in(0,H)$ and $\delta\in(0,1)$, and a PWC $\mathcal{C}$ satisfying Assumption~\ref{assp:constraint}, if
\begin{small}
\begin{equation}
   N\geq \frac{8H^4\ln(8SA/\delta)}{\epsilon^2}(2+\lambda\sqrt{D})^2,
\end{equation}
\end{small}
then with a probability no less than $1-\delta$, 
\begin{equation}
    Err(\widehat{\pi}^*)\leq\epsilon.
\end{equation}
\end{theorem}
The detail of the proof is given in Appendix~\ref{app:pwc_thm}. This theorem shows that compared with Theorem~\ref{thm:main_thm}, the sample complexity of PWCs for model estimation can be improved from $\widetilde{O}(DSAH^4/\epsilon^2)$ to $\widetilde{O}(SAH^4/\epsilon^2)$, while the bound for perturbation set error still suffers a scale of $D\lambda^2$. 

We also give more results for specific PWC cases. For RMDP with fixed perturbation sets, where the perturbation is independent of the environment parameters, we have $Err(\widehat{\pi}^*) \leq \epsilon$ with a probability no less than $1-\delta$, with $N\geq \frac{8H^4\ln(8SA/\delta)}{\epsilon^2}$. Notice that this sample complexity is comparable to the results of solving MDPs without perturbation, which has an order of $\widetilde{O}(SAH^3\min(H,S)/\epsilon^2)$ (See Sec. C.2 of \citet{cui2020plug}). In other words, if our estimation of $\widehat{P}$ would not cause a perturbation difference, the complexity for solving a RMDP is comparable to that of solving an MDP.

Further we consider another kind of PWC constraint, the Total Variation Distance Constraint (TVDC). For a fixed value $u$, for each state-action pair $(s,a)$ and any probability $p\in\Delta_{sa}$, the TVDC is defined as
\begin{equation}
	\mathcal{C}_{TVD}(p;u):=\{p'\in\Delta_{sa}:||p'-p||_{TV}\leq u\}.
\end{equation}

We also aim to ensure $Err(\widehat{\pi}^*) \leq \epsilon$ with a probability no less than $1-\delta$.  If $u\leq \epsilon/(16H^2)$, then we can choose $N\geq \frac{128H^4\ln(8SA/\delta)}{\epsilon^2}$; otherwise we choose $N\geq \max\left(\frac{16H^4\ln(8SA/\delta)}{\epsilon^2}(\sqrt{2}+6\sqrt{uD})^2,\frac{49D\ln(8SA/\delta)}{27u}\right).$

The complete theorem and its proof are given in Appendix~\ref{app:tvdc_thm}. We can see that for a small $u$, we can nearly consider the TVDC as a fixed perturbation constraint and it has the sample complexity order comparable to that of solving a MDP without perturbation. For $u>\epsilon/(16H^2)$, the perturbation set difference would cause a scale of $1+uD$ on the sample complexity. That is, the larger $u$ is, the larger the sample complexity is needed to guarantee the $Err(\widehat{\pi}^*)$. Finally, there is a higher order term of $49D\ln(8SA/\delta)/(27u)\leq 49*16H^2D\ln(8SA/\delta)/(27\epsilon)$.

\subsection{Homogeneous perturbation}

PWC problems can be solved by minimax dynamic programming, which is quite similar to previous methods. Here we give a simple example to show that our game-theoretical framework is able to handle more flexible constraints. Here we consider a homogeneous perturbation (HP), where the perturbation on all state-action pairs should be the same. Here we simply assume that $\Delta_{sa}$ has the same dimension for all $(s,a)$ and perturbation $\sigma$ from a fixed perturbation set $U$ is added on all state-action pairs simultaneously. For HP, we cannot choose a worst-case perturbation for each $(s,a)$, and we need to consider the MDP parameters as a whole.

This problem cannot be solved by a directly minimax optimization on state-action pairs. From our game-theoretical point of view, our two-player zero-sum game defined in Sec.~\ref{sec:tzg} reduce to a two-player zero-sum extensive game with imperfect information (TZEGI) (Refer Sec. 3.7 of \citet{2007Algorithmic} and Def 1 of \citet{zinkevich2007regret}). In TZEGI, players make decisions on infosets, which includes all the states that the player cannot distinguish. Here we consider the perturbation play takes action the infoset which includes all state-action pairs. Therefore we only need to find an NE solver for TZEGI. Fortunately \citet{zinkevich2007regret} provide us an approximate NE solver, the counterfactual regret minimization (CFR)\footnote{Notice that CFR usually returns a random policy.}. We can use CFR as the solver to return an $\epsilon/4$-NE and choose $N=O(H^4D\ln(1/\delta)/\epsilon^2)$.

Hence our game-theoretical framework indeed provide us a way to handle RMDP with more general constraints. More detailed discussion for this is given in Appendix~\ref{app:hp}.

\subsection{Extension to Robust Partial Observable MDP}

Our method is even able to extend to robust Partial Observable MDP (RPOMDP) problems, where the agent is given observations instead of states. 

Consider an RPOMDP problem with the training environment parameters give. The key to solve an RPOMDP is to formalize this problem as a TZEGI, where the agent's infosets includes all the histories it cannot distinguish now. Again we apply the CFR algorithm to find the robust policy for this task. We give detailed discussion in Appendix~\ref{app:rpomdp}.

\section{Discussions}

Here we give some discussions for our results. 

\textbf{Rectangularity Assumption (RA):} many previous works \cite{yang2021towards} consider RA, which is also a PWC and is appliable with Thm.~\ref{thm:pwc_thm}. The standard robust dynamic programming technique is suitable for RA, but cannot be used under our weak Lipschitz assumption.

\textbf{Lower bound:} \citet{azar2013minimax} give the lower bound for solving an episodic MDP with a generative model of order $O(SAH^3\ln(SA/\delta)/\epsilon^2)$. This can be considered as a lower bound for RMDP. However, it is not clear whether the lower bound can be improved according to the property of the constraints. For the fixed perturbation constraint and the TVDC with small $u$, our method matches this lower bound except an extra $H$. For the general constraint with a Lipschitz constant $\lambda$, our upper bound matches the lower bound except an extra scale of $HD(1+\lambda)^2$.

\textbf{Online setting:} In the online setting, the agent interacts with the environment to gain trajectories and uses the trajectories to learn policy. The lower bound for online and generative model settings are the same \cite{azar2013minimax}. However, for the RMDP problem in this work, the online methods might suffer a sample complexity growing exponentially with $H$ or even infinity. The key challenge is that the agent in the online problem can only reach states according to $P^s$. For example, consider the case that the reaching probability to a state $s$ under MDP $M^s$ and any policy $\pi$ is 0, while it can be reached under the true environment $M^*$. Then the agent can never know the reward at and after $s$, and this might cause a large error. It is still not clear whether we can solve this issue for the online setting.

\textbf{Comparison with MDP without perturbations:} Compared to current results of $\widetilde{O}(SAH^3\min(H,S)/\epsilon^2)$ in \citet{cui2020plug} for finite-horizon MDPs without perturbation, our result in Theorem~\ref{thm:main_thm} suffers an extra $D(1+\lambda)^2$ term. The term $(1+\lambda)^2$ comes from the Lipschitz condition of $\mathcal{C}$, and the $D$ appears as we require transitions for all state-action pair $(s,a)$ to be tight enough. The appearance of $D$ is similar to the conclusion of the recently studied reward-free reinforcement learning setting \citep[e.g.][]{jin2020reward, zhang2020nearly}.

\section{Related Work}
Many existing works concentrate on the training and testing mismatch problem on MDPs. A straightforward way is to connect the two MDPs and transfer knowledge from $M^s$ to $M^*$. \citet{rusu2017sim} apply transfer learning to transfer features for testing. \citet{jiang2018pac} theoretically gives some results on how to repair a mismatched training MDP.

As for the problems where the testing environment is not accessible, there are also various works focusing on solving the robust policies via maximin methods. \citet{dupavcova1987minimax} discusses the minimax approaches for stochastic problems and \citet{shapiro2002minimax} give a Bayesian approach.  \citet{nilim2004robustness} and \citet{iyengar2005robust} propose a general framework for the maximin solutions for RMDPs. However, they require full knowledge of $M^s$ and the perturbations. \citet{wiesemann2013robust} concentrate on RMDP under RA and \citet{puggelli2013polynomial} focus on convec uncertainties. Recent work proposes a model-free method for RMDPs \cite{roy2017reinforcement} and the extension to MDPs with functional approximations is also considered~\cite{tamar2014scaling,panaganti2020modelfree}. \citet{derman2020bayesian} propose a Bayesian method for RMDPs. Although they consider the problems where the interactions with $M^s$ are needed, they lack non-asymptotic analysis for the sample complexity. Some further work \cite{cubuktepe2021robust} considers RPOMDPs under certain assumptions.

As for the solving an MDP with a generative model, plenty of works \cite{kearns1999finite,kakade2003sample,azar2012sample,azar2013minimax,cui2020plug} have been done for the sample complexity to find an optimal policy. They mostly concentrate on infinite-horizon MDP with a discount factor $\gamma$ and currently have reached the lower bound of $\widetilde{O}(SA/((1-\gamma)^3\epsilon^2)$. However, their results cannot be applied to RMDPs directly since the perturbation makes the environment for the agent no longer stable.

\section{Conclusion}

We focus on the policy learning for Robust Markov Decision Process where the training and testing MDPs are mismatched. For general constraints on the perturbation, we solve RMDP problems by re-formalizing the problem as finding an NE of a two-player zero-sum game and applying a plug-in NE solver. For a constraint satisfying the Lipschitz condition with a constant $\lambda$, our method has a sample complexity of $\widetilde{O}(DSAH^4(1+\lambda)^2/\epsilon^2)$. Our method can solve constraints like PWC and HP and is able to extend to solve problems like RPOMDPs. 

\section{Acknowledgments}

This work was supported by the National Key Research and Development Program of China (No.s 2020AAA0106000, 2020AAA0104304, 2020AAA0106302), NSFC Projects (Nos. U19A2081, 61620106010, 62061136001, 62076147, U19B2034, U1811461), Beijing NSF Project (No. JQ19016), Beijing Academy of Artificial Intelligence (BAAI), Tsinghua-Huawei Joint Research Program, and Tsinghua Institute for Guo Qiang.

\bibliography{mybib}

\newpage
\appendix
\onecolumn

\section{Notations}
	
We list our notations in this section.
\\
\\
\begin{tabular}{ll}
\hline
     Notation & Explanation \\
     \hline
    $M$ & An Markov Decision Process (MDP). \\
    $\mathcal{S}$ & The set of states. \\
    $\mathcal{S}_h$ & The set of states at depth $h$. \\
    $S$ & The number of total states. \\
    $\mathcal{A}$ & The set of actions. \\
    $A$ & The number of total actions. \\
    $P$ & The transition function. \\
    $P(s,a)$ & The transition vector at state-action pair $(s,a)$. \\
    $P(s'|s,a)$ & The probability to transit to $s'$ from state-action pair $(s,a)$. \\
    $r$ & The deterministic reward function. \\
    $H$ & The horizon of the MDP. \\
    $s_0$ & The initial state at depth $1$.\\
    $s_{\tau}$ & The terminal state at depth $H+1$. \\
    $D$ & $D=\max_{h\in[H]}|\mathcal{S}_h|$. \\
    $G$ & $G=\sum_{h\in[H]}|\mathcal{S}_h|$. \\
    $\Delta_{sa}$ & The probability simplex for $P(s,a)$. \\
    $\pi$ & The policy of the agent, mapping each state to an action. \\
    $V^{\pi}_{M}(s)$ & The expected total reward from state $s$ at depth $h$ following $\pi$ under $M$. \\
    $Q^{\pi}_{M}(s,a)$ & The expected total reward from state-action pair $(s.a)$ at depth $h$ following $\pi$ under $M$.. \\
    $V^{\pi}_{M.h}$ & The vector for of $V^{\pi}_{M}(s)$ for $s\in\mathcal{S}_h$. \\
    $M^s$ & The training MDP. \\
    $M^*$ & The testing MDP. \\
    $\mathcal{P}$ & The set of all possible transitions, i.e. $\prod_{(s,a)}\Delta_{D_{sa}}$. \\
    $\mathcal{C}$ & The constraint, mapping each $P\in\mathcal{P}$ to a subset of $\mathcal{P}$. \\
    $U$ & The perturbation set, $U(P)=\mathcal{C}(P)-P$. \\
    $\widetilde{V}^\pi_{h}$ & The worst-case $V$. \\
    $\pi^*$ & The optimal robust policy. \\
    $Err(\pi)$ & The error value between $\pi$ and $\pi^*$, defined in Eq.~\ref{eq:pi_error}.\\
    $N$ & The number of samples for each state-action pair. \\
    $\widehat{P}$ & The empirical estimated transition. \\
    $\widehat{M}$ & The empirical estimated MDP. \\
    $\sigma$ & The elements of $U(P)$, considered as the policy of player 2. \\
    $\sigma(s,a)$ & The perturbation added to $P(s,a)$. \\
    $\sigma(s'|s,a)$ & The perturbation added to $P(s'|s,a)$. \\
    $P+\sigma$ & The perturbed transition. \\
    $M+\sigma$ & The perturbed MDP. \\
    $\mathcal{V}^{\pi,\sigma}_{M}(s)$ & The expected value for player 1 at state $s$ of depth $h$, under MDP $M$ and policy pair $(\pi,\sigma)$.\\
    $\mathcal{Q}^{\pi,\sigma}_{M}(s,a)$ & The expected value for player 1 at $(s,a)$ of depth $h$, under MDP $M$ and policy pair $(\pi,\sigma)$.\\
    $\mathcal{C}_{sa}$ & The constraint on state-action pair $(s,a)$. \\
    $U_{sa}$ & The perturbation on state-action pair $(s,a)$. \\
    $\lambda$ & The Lipschitz constant in Assumption \ref{assp:pwc}. \\
    $(\pi^*,\sigma^*)$ & The NE solution under $M^s$ and $U(P^s)$.\\
    $(\widehat{\pi}^*,\widehat{\sigma}^*)$ & The NE solution under $\widehat{M}$ and $U(\widehat{P})$.\\
    $\sigma',\widetilde{\sigma}^*,\widetilde{\sigma}'$ & Player 2's policy defined in Eq.~\eqref{eq:sigma_prime} and \eqref{eq:sigma_tilde_prime}.\\
    $\xi_{M}^{\pi,\sigma}(s,a)$ & The probability of reaching $(s,a)$ under $M$ following $(\pi,\sigma)$. \\
    $\alpha,\beta$ & Functions defined in Eq.~\eqref{eq:alpha_def} and \eqref{eq:beta_def}.\\
\hline 
\end{tabular}
	
	\subsection{Some furhter notations for sample complexity analysis}
	
	\begin{itemize}
		\item We define $\mathcal{E}_G$ as the good event for general constraints.
		\item We denote $Var_{P(s,a)}(\mathcal{V})$ for $s\in\mathcal{S}_{h-1}$ as the variance over the transition function $P(s,a)$ as
		\begin{equation}
		Var_{P(a,s)}(\mathcal{V})=\sum\limits_{s'\in\mathcal{S}_{h}}P(s'|s,a)\left(\mathcal{V}(s')-\sum\limits_{s''\in\mathcal{S}_h}P(s''|s,a)\mathcal{V}(s'')\right)^2.
		\end{equation}
	\end{itemize}

    \section{Pair-Wise Constraints}
    \label{sec:pwc_results}
    
    Here we give a specific kind of constraints of RMDP, where the perturbations are independent among state-action pairs. 
    
    \subsection{Pair-Wise Constraint (PWC)} 
    \label{sec:pwc_def}
    
    Here we consider a specific kind of constraints, where the constraint $\mathcal{C}$ has independent effect on each state-action pair. We call this as the \textit{Pair-Wise Constraint} (PWC). Many practical problems provide constraints directly on state-action pairs, such as Total Variation distance bound on transitions. These constraints can be included in the PWC.
    
    Formally, a PWC constraint $\mathcal{C}$ satisfies that:
    \begin{assumption}
    	\label{assp:pwc}
    	For any state action pair $(s,a)$, there is a function $\mathcal{C}_{sa}$ that maps each $p\in\Delta_{D_{sa}}$ to a subset of $\Delta_{D_{sa}}$ such that for any transition function $P\in\mathcal{P}$, 
    	\begin{equation}
    	\mathcal{C}(P)=\prod_{(s,a)}\mathcal{C}_{sa}(P(s,a)).
    	\end{equation}
    \end{assumption}
    	Under the PWC, the worst-case value $\widetilde{V}_h^\pi$ satisfies the robust Bellman equation as 
    \begin{small}
    	\begin{align}
    	\label{eq:pwc_bellman}
    	&\widetilde{V}^\pi(s) =\min_{p\in\mathcal{C}_{sa}(P^s(s,a))}\left[r(s,a)+\sum\limits_{s'}p(s')\widetilde{V}^\pi(s')\right],
    	\end{align}
    \end{small}
    where $P(s_{\mathcal{T}}|s,a)=1$ for $s\in\mathcal{S}_H$ and $\widetilde{V}^\pi_{H+1}(s_\mathcal{T})=0$.
    
    \subsection{An NE Solver for PWC}
    \label{sec:pwc_ne_solver}
    In Sec.~\ref{sec:tzg}, we assume to have an NE solver $\mathcal{O}$. The implementation of the solver relies on the forms of the constraint $\mathcal{C}$. If $\mathcal{C}$ is a PWC, we can give a simple NE solver $\mathcal{O}_{PW}$.
    
    According to Assumption~\ref{assp:pwc}, a PWC satisfies that $\mathcal{C}(P)=\prod_{(s,a)}\mathcal{C}_{sa}(P(s,a))$ for any $P\in\mathcal{P}$. Therefore, the perturbation set $U$ induced by a PWC $\mathcal{C}$ is also independent among $(s,a)$. We can similarly decompose the set $U(P)$. We define a perturbation set for state-action pair $(s,a)$ as
    \begin{equation}
        U_{sa}(P):=\{p\in[-1,1]^{D_{sa}}: p+P(s,a)\in\mathcal{C}_{sa}(P(s,a))\}.
    \end{equation}
    We call a perturbation set $U$ induced by a PWC $\mathcal{C}$ as a \textit{Pair-Wise Perturbation} (PWP) and it satisfies $U(P)=\prod_{(s,a)}U_{sa}(P)$.
    
    For $\sigma\in U(P)$, we use $\sigma(s,a)$ to represent the perturbation vector added to $P(s,a)$. For a PWP $U$, $\sigma(s,a)$ can be calculated independently. 
    
    With the independence of $\sigma(s,a)$ and the Bellman Equation \eqref{eq:pwc_bellman} for PWC problems, we are now ready to design our NE solver $\mathcal{O}_{PW}$. Formally, the input is an MDP $M=\langle\mathcal{S},\mathcal{A},P,r,H\rangle$ and a PWP $U(P)=\prod_{(s,a)}U_{sa}(P)$. The output policy pair $(\pi',\sigma')$ can be calculated via back-propagation from depth $H$ back to $1$.
    
    First, we consider the terminal state $s_{\tau}$ and define ${\mathcal{V}}^{\pi',\sigma'}_{M}(s_{\tau})=0$. Specifically, we can directly set ${\sigma}'(s,a)=\bm{0}$ if $s\in\mathcal{S}_H$. For depth $h+1$ ($h\in [H]$), assume that we have already calculated ${\mathcal{V}}^{\pi',\sigma'}_{M}$. For state $s\in\mathcal{S}_h$ and action $a$, we simply solve for ${\sigma}'(s,a)$ as
    \begin{small}
    	\begin{align}
    	\label{eq:delta_p}
    	{\sigma}'(s,a)=\mathop{\arg\min}_{p\in U_{sa}(P)}\sum_{s'}({P}(s'|s,a)+p(s')){\mathcal{V}}^{\pi',\sigma'}_{M}(s').
    	\end{align}
    \end{small}
    To calculate ${\mathcal{V}}^{\pi',\sigma'}_{M}$, we work out the ${\mathcal{Q}}^{\pi',\sigma'}_{M}$ values as
    \begin{small}
    	\begin{align*}
    	{\mathcal{Q}}^{\pi',\sigma'}_{M}(s,a)=&r(s,a) +\sum_{s'}({P}+{\sigma}')(s'|s,a){\mathcal{V}}^{\pi',\sigma'}_{M}(s').
    	\end{align*}
    \end{small}
    Finally, we calculate the robust policy $\widehat{\pi}^*$ and its value with:
    \begin{align*}
    {\pi}'(s)=&\argmax_a {\mathcal{Q}}^{\pi',\sigma}_{M}(s,a),\\
    {\mathcal{V}}^{\pi',\sigma'}_{M}(s)=&\max_a{\mathcal{Q}}^{\pi',\sigma}_{M}(s,a).
    \end{align*}
    
    Thus $(\pi',\sigma')$ is returned as the output of $\mathcal{O}_{PW}(M,U(P))$.
    
    The idea of solving a minimax optimization have been used in previous works~\cite{nilim2004robustness,iyengar2005robust,roy2017reinforcement}. The key advantage of formalizing this robust issue into a game-theoretical framework is that we can solve more general and complex problems with a NE solver. In next section, we give theoretical results for general problems and give examples of some specific tasks.
    
    The results for a PWC satisfying Assumption~\ref{assp:constraint} can be further improved. Using the PWC NE solver $\mathcal{O}_{PW}$ in Sec.~\ref{sec:pwc_ne_solver}, we can upper bound $(P^s(s,a)-\widehat{P}(s,a))^\top \mathcal{V}^{{\pi}^*,\widehat{\sigma}^*}_{\widehat{M},h+1}$ in Eq.~\eqref{eq:e1_p_gap} directly. The reason we can do this is that $\mathcal{O}_{PW}$ solves $\widehat{\sigma}^*(s,a)$ for $s\in\mathcal{S}_h$ without information from depth $h'<h$. Using the Bernstein's inequality (refer to Lemma~\ref{lem:bernstein} in Appendix~\ref{app:lemmas}), with a probability at least $1-\delta'$, we have
    \begin{equation}
        (P^s(s,a)-\widehat{P}(s,a))^\top \mathcal{V}^{{\pi}^*,\widehat{\sigma}^*}_{\widehat{M},h+1}\leq \beta(N),
    \end{equation}
    where
    \begin{equation}
        \label{eq:beta_def}
        \beta(N):=H\sqrt{\frac{2\ln(2/\delta')}{N}}+\frac{H\ln(2/\delta')}{3N}.
    \end{equation}
    
    Using similar techniques as in Sec.~\ref{sec:main_results}, we can get the following theorem.
    \begin{theorem}
    \label{thm:pwc_thm_copy}
    For a given $(\epsilon, \delta)$, where $\epsilon\in(0,H)$ and $\delta\in(0,1)$, and a PWC $\mathcal{C}$ satisfying Assumption~\ref{assp:constraint}, if
    \begin{small}
    \begin{equation}
      N\geq \frac{8H^4\ln(8SA/\delta)}{\epsilon^2}(2+\lambda\sqrt{D})^2,
    \end{equation}
    \end{small}
    then with a probability no less than $1-\delta$, 
    \begin{equation}
        Err(\widehat{\pi}^*)\leq\epsilon.
    \end{equation}
    \end{theorem}
    
    The detail of the proof is given in Appendix~\ref{app:pwc_thm}. This theorem shows that compared with Theorem~\ref{thm:main_thm}, the sample complexity of PWCs for model estimation can be improved from $\widetilde{O}(DSAH^4/\epsilon^2)$ to $\widetilde{O}(SAH^4/\epsilon^2)$, while the bound for perturbation set error still suffers a scale of $D\lambda^2$. 
    
    \subsection{Fixed Perturbation Constraint}
    \label{sec:fix_perturbations}
    The simplest case for a PWC $\mathcal{C}$ is the fixed perturbation constraint. That is, $U(P)=\mathcal{U}$ for all possible $P$. We now show that this can lead to a simple sample complexity result, which is comparable to that of solving MDP without perturbations.
    
    \begin{corollary}
    	\label{thm:fix_perturb_thm}
    	For a given $(\epsilon, \delta)$, where $\epsilon\in(0,H)$ and $\delta\in(0,1)$, and a fixed set $\mathcal{U}$, if a PWC $\mathcal{C}$ satisfies that
    	$U(P)=\mathcal{U}$ for any tractable $P$ and
    	\begin{align}
        N\geq \frac{8H^4\ln(8SA/\delta)}{\epsilon^2},
    	\end{align}
    	then with a probability no less than $1-\delta$, $\widehat{\pi}^*$ satisfies
    	\begin{align}
    	Err(\widehat{\pi}^*) \leq \epsilon.
    	\end{align}
    \end{corollary}	
    Since the perturbation sets are the same for $P^s$ and $\widehat{P}$, we can check our Assumption~\ref{assp:constraint} and set $\lambda=0$. Then this corollary is a simple extension of Theorem~\ref{thm:pwc_thm}. Notice that this sample complexity is comparable to the results of solving MDPs without perturbation, which has an order of $\widetilde{O}(SAH^3\min(H,S)/\epsilon^2)$ (See Sec. C.2 of \citet{cui2020plug}). In other words, if our estimation of $\widehat{P}$ would not cause a perturbation difference, the complexity for solving a RMDP is comparable to that of solving an MDP.
    
    \subsection{Total Variation Distance Constraint}
    \label{sec:tvdc}
    Here we consider another specific PWC constraint, the Total Variation Distance Constraint (TVDC). Consider a fixed small value $u$. For each state-action pair $(s,a)$ and any probability distribution $p\in\Delta_{D_{sa}}$, the TVDC is defined as
    \begin{equation}
    	\mathcal{C}_{TVD}(p;u):=\{p'\in\Delta_{D_{sa}}:||p'-p||_{TV}\leq u\}.
    \end{equation}
    
    Since we only consider the Categorical distribution $p,p'\in\Delta_{D_{sa}}$, thus we have that
    \begin{align}
    	||p'-p||_{TV}= \frac{1}{2}||p'-p||_1
    	= \frac{1}{2}\sum_{i=1}^{D_{sa}}|p'(i)-p(i)|.
    \end{align}
    
    It is easy to set $\lambda=2$ for TVDC and get a sample complexity of $\widetilde{O}(H^4D/\epsilon^2)$, which is a loose bound for TVDC.
    
    We can consider the TVDC as a maximum value of the 1-norm of the perturbation. Therefore, if $u$ is too small to cause an error of $\epsilon/2$, we can simply remove it from $\epsilon$. More specifically, if $u\leq \epsilon/(16H^2)$, we have that the value of Eq.~\eqref{eq:e1_sigma_gap} is no larger than $\epsilon/4$. For this case, we can simply choose an $N$ to ensure that Eq.~\eqref{eq:e1_p_gap} is not larger than $\epsilon/4$.
    
    When $u>\epsilon/(16H^2)$, we can still tight the sample complexity bound with a careful analysis of the property of TVDC. For this TVDC, we can give the following theorem.
    
    \begin{theorem}
    	\label{thm:tvdc_thm}
    	Consider a given $(\epsilon, \sigma)$, where $\epsilon\in(0,H)$ and $\delta\in(0,1)$, and a TVDC $\mathcal{C}_{TVD}(\cdot;u)$. If $u\leq \epsilon/(16H^2)$, then using
    	\begin{equation}
    	    N\geq \frac{128H^4\ln(8SA/\delta)}{\epsilon^2},
    	\end{equation}
    	with a probability at least $1-\delta$ we have 
    	\begin{equation}
    	    Err(\widehat{\pi}^*) \leq \epsilon.
    	\end{equation}
    	If $u>\epsilon/(16H^2)$, then using
    	\begin{small}
    	\begin{align}
    	N\geq&  \max\left(\frac{16H^4\ln(8SA/\delta)}{\epsilon^2}(\sqrt{2}+6\sqrt{uD})^2,\frac{49D\ln(8SA/\delta)}{27u}\right),
    	\end{align}
    	\end{small}
    	then with a probability no less than $1-\delta$, $\widehat{\pi}^*$ satisfies
    	\begin{align}
    	Err(\widehat{\pi}^*) \leq \epsilon.
    	\end{align}
    \end{theorem}
    
    The complete proof is given in Appendix~\ref{app:tvdc_thm}. From the above theorem, we can see that for a small $u$, we can nearly consider the TVDC as a fixed perturbation constraint and it has the sample complexity order comparable to that of solving a MDP without perturbation. For $u>\epsilon/(16H^2)$, the perturbation set difference would cause a scale of $1+uD$ on the sample complexity. That is, the larger $u$ is, the larger the sample complexity is needed to guarantee the $Err(\widehat{\pi}^*)$. Finally, there is a higher order term of $49D\ln(8SA/\delta)/(27u)\leq 49*16H^2D\ln(8SA/\delta)/(27\epsilon)$.
	
	\section{Homogeneous Perturbation}
	\label{app:hp}
	
	Here we consider a homogeneous perturbation (HP), where the perturbation on all state-action pairs should be the same. First we give a relatively formal formulation for this problem. Here we simply assume that $\Delta_{sa}=|\mathcal{S}_h|=S/H$ for all $h\in[H],s\in\mathcal{S}_h$. That is, $D_{sa}$ has the same dimension for all $(s,a)$. Then the perturbation $\sigma$ is a $S/H$-dimensional vector from a fixed perturbation set $U$. For all $(s,a)$, the perturbed transition is now $P(s,a)+\sigma$. Therefore all state-action pairs are perturbed with the same vector simultaneously. For HP, we cannot choose a worst-case perturbation for each $(s,a)$, and we need to consider the MDP parameters as a whole.

    This problem cannot be solved by a directly minimax optimization on state-action pairs. This is because the choice of perturbation holds for all $(s,a)$ . From our game-theoretical point of view, we can consider that the opponent decision its action without knowing the current states of the agent. Hence the RMDP problem with HP can be formulated as a two-player zero-sum extensive game with imperfect information (TZEGI) (Refer Sec. 3.7 of \citet{2007Algorithmic} and Def 1 of \citet{zinkevich2007regret}). 
    
    \subsection{Two-player Zero-sum Game with Imperfect information}
    Here we give a formal definition for TZEGI and one algorithm for it.
    
    \begin{Def}
    A Two-player Zero-sum Extensive Game with Imperfect information are composed of below elements:
    \begin{itemize}
        \item An agent which takes action from $\mathcal{A}$ and an opponent that takes action from $U$. The environment transition, i.e., $P^s$ is considered as a chance player. That is, the chance player takes $s'$ as its action at state-action pair $(s,a)$ with a probability $P^s(s'|s,a)$.
        \item The history set $\mathcal{H}$ which is composed of a sequence of actions (including the actions of the chance player). Each history represent a possible sequence that can appear during the game. If a history $f=[a_1, a_2, ..., a_t]$ is in $\mathcal{H}$, then for all $t'\leq t$, history $f'=[a_1,a_2,...,a_{t'}]$ is also in $\mathcal{H}$.
        \item At the end of one game, a history $z\in\mathcal{Z}$ is generated and a corresponding reward $v(z)$ is returned. The payoff of the agent for this game is $v(z)$ and that of the opponent is $-v(z)$. Here $\mathcal{Z}$ is the set of all terminal histories.
        \item An infoset $I$ is a set of histories that the current player cannot distinguish according to its observations. That is, for histories $f,f'\in I$, current player has the same observation for them.
        \item The policy of each player is defined on its infosets, since they cannot distinguish these histories. Notice that the policy here refers to random policy.
    \end{itemize}
    \end{Def}
    
    In TZEGI, histories are used to represent states, since information is imperfect and the Markov property is no longer suitable.
    
    Counterfactual Regret Minimization (CFR) \cite{zinkevich2007regret} can solve the approximated NE of a TZEGI by repeatedly play the game. More detail can be found in \citet{zinkevich2007regret}. 
    
    \subsection{Solving RMDP with HP}
    
    We need to re-formalize the RMDP problem into a TZEGI. A simple we to construct such a TZEGI is to discretize the action space of the opponent and the chance player such that the whole RMPD can be represented by a decision tree. With such a tree, we let the infoset for the opponent to be all its decision histories. That is, the opponent directly makes its decision and apply it to all its decision histories. Then we can apply CFR to solve an approximate NE.
    
    According to Theorem 4 of \citet{zinkevich2007regret}, we have that with $T$ episodes of games using CFR, we can get an $\epsilon$-NE with $\epsilon$ has an order of $O(1/\sqrt{T})$. 
    
    It need to be noticed that the game tree can be quite large due to our use of histories and discretization. Thus the space and time complexity for this NE solver can be large. If possible we can also apply some other plug-in solver to give approximated NEs.
    
    If we have get an $\epsilon/4$-NE, then the key Lemma 1 becomes
    \begin{equation}
        Err(\hat{\pi}^*)\leq e_1+e_2+\epsilon/2.
    \end{equation}
	
	Thus we can use the same order of samples as Theorem 1.
	
	\section{Robust Partial-Observable MDP}
	\label{app:rpomdp}
	
	Our method is even able to extend to robust Partial Observable MDP (RPOMDP) problems, where the agent is given observations instead of states. 

    In POMDP, when the environment is at state $s$, it returns agent an observation $o$ with a probability $\Omega(o|s)$. Since the learning process for POMDPs can be much more complex than that of MDPs, here we only consider the case that the environment parameters of the POMDP is given. Then we can directly build this problem into a history tree where the observation sequences can corresponds to the infoset for the agent.
	
	\section{Lemmas}
	\label{app:lemmas}
	\begin{lemma}(Berstein's inequality~\cite{maurer2009empirical})
		\label{lem:bernstein}
		Let $Z,Z_1,Z_2,..., Z_n$ to be i.i.d. random variables which are bounded in $[0,C]$. Let $\delta' \in (0,1)$. With a probability no less than $1-\delta'$, we have that 
		\begin{equation}
		\left|\mathbb{E}Z-\frac{1}{n}\sum\limits_{i=1}^nZ_i\right|\leq \sqrt{\frac{2Var(Z)\ln(2/\delta')}{n}} + \frac{C\ln(2/\delta')}{3n},
		\end{equation}
		where $Var(Z)$ is the variance of random variable of $Z$.
	\end{lemma}
	The proof for this lemma can be found in Theorem 3 of \citet{maurer2009empirical}.
	
	For a given $\delta\in(0,1)$ and $\delta'=\delta/(4SA)$, we define a good event as
	\begin{align}
	\mathcal{E}_G := & \left\{\forall (s,a),\ ||P^s(s,a)-\widehat{P}(s,a)||_1 \leq \sqrt{\frac{2D\ln(2/\delta')}{N}}\right\}.
	\end{align}
	
	\begin{lemma}
		\label{lem:good_event}
		With a probability no less than $1-\delta$, the good event $\mathcal{E}_G$ holds.
	\end{lemma}
	
	This Lemma holds due to the Azuma's inequality and the union bound.
	
	For the convenience of value decomposition, we denote $\xi^{\pi,\sigma}_M(s,a)$ as the reaching probability to state-action pair $(s,a)$ following policy pair $(\pi,\sigma)$ under MDP $M$.
	\begin{lemma}
		\label{lem:v_gap_decompose}
		For two MDPs $M=\langle \mathcal{S}, \mathcal{A}, P, r, H\rangle$ and $M'=\langle \mathcal{S}, \mathcal{A}, P', r, H\rangle$, given the player 1's policy $\pi$ and the player 2's policy $\sigma$ where $\sigma\in U(P)\cap U(P')$, then 
		\begin{align}
		\mathcal{V}^{\pi,\sigma}_{M}(s_0)-\mathcal{V}^{\pi,\sigma}_{M'}(s_0) = & \sum_{h=1}^H\sum_{s\in \mathcal{S}_h}\sum_{a}{\xi}^{\pi,\sigma}_{M}(s,a)\left( (P(s,a)-P'(s,a))^\top \mathcal{V}^{\pi,\sigma}_{M',h+1}\right),\\
		\mathcal{V}^{\pi,\sigma}_{M}(s_0)-\mathcal{V}^{\pi,\sigma}_{M'}(s_0) = & \sum_{h=1}^H\sum_{s\in \mathcal{S}_h}\sum_{a}{\xi}^{\pi,\sigma}_{M'}(s,a)\left( (P(s,a)-P'(s,a))^\top \mathcal{V}^{\pi,\sigma}_{M}\right).
		\end{align}
	\end{lemma}
	\begin{proof}
		We have that
		\begin{align*}
		& \mathcal{V}^{\pi,\sigma}_{M}(s_0)-\mathcal{V}^{\pi,\sigma}_{M'}(s_0) \\
		= & \mathcal{Q}^{\pi,\sigma}_{M}(s_0,\pi(s_0)) - \mathcal{Q}^{\pi,\sigma}_{M'}(s_0,\pi(s_0))\\
		= & (P+\sigma)(s_0,\pi(s_0))^\top \mathcal{V}^{\pi,\sigma}_{M,2} -(P'+\sigma)(s_0,\pi(s_0))^\top \mathcal{V}^{\pi,\sigma}_{M',2}\\
		= & (P + \sigma)(s_0,\pi(s_0))^\top (\mathcal{V}^{\pi,\sigma}_{M,2} - {\mathcal{V}}^{\pi,\sigma}_{M',2}) + (P-P')(s_0,\pi(s_0))^\top \mathcal{V}^{\pi,\sigma}_{M',2}\\
		= & \sum_{h=1}^H\sum_{s\in \mathcal{S}_h}\sum_{a}{\xi}^{\pi,\sigma}_{M}(s,a)\left( (P(s,a)-P'(s,a))^\top \mathcal{V}^{\pi,\sigma}_{M',h+1}\right).
		\end{align*}
		
		Similarly, we can prove that
		\begin{align*}
		& \mathcal{V}^{\pi,\sigma}_{M}(s_0)-\mathcal{V}^{\pi,\sigma}_{M'}(s_0) \\
		= & \mathcal{Q}^{\pi,\sigma}_{M}(s_0,\pi^*(s_0)) - \mathcal{Q}^{\pi,\sigma}_{M'}(s_0,\pi(s_0))\\
		= & (P+\sigma)(s_0,\pi(s_0))^\top \mathcal{V}^{\pi,\sigma}_{M,2} -(P'+\sigma)(s_0,\pi(s_0))^\top \mathcal{V}^{\pi,\sigma}_{M',2}\\
		= & (P' + \sigma)(s_0,\pi(s_0))^\top (\mathcal{V}^{\pi,\sigma}_{M,2} - {\mathcal{V}}^{\pi,\sigma}_{M',2}) + (P-P')(s_0,\pi(s_0))^\top \mathcal{V}^{\pi,\sigma}_{M,2}\\
		= & \sum_{h=1}^H\sum_{s\in \mathcal{S}_h}\sum_{a}{\xi}^{\pi,\sigma}_{M'}(s,a)\left( (P(s,a)-P'(s,a))^\top \mathcal{V}^{\pi,\sigma}_{M}\right).
		\end{align*}
		Then we finish the proof.
	\end{proof}
	
	\begin{lemma}
		\label{lem:v_gap_decompose_gen}
		For two MDPs $M=\langle \mathcal{S}, \mathcal{A}, P, r, H\rangle$ and $M'=\langle \mathcal{S}, \mathcal{A}, P', r, H\rangle$, given the player 1's policy $\pi$ and the player 2's policy $\sigma$ where $\sigma\in U(P)$, and another player 2's policy $\widetilde{\sigma}\in U(P')$, 
		then we have that
		\begin{align}
		\mathcal{V}^{\pi,\sigma}_{M}(s_0)-\mathcal{V}^{\pi,\widetilde{\sigma}}_{M'}(s_0) = & \sum_{h=1}^H\sum_{s\in \mathcal{S}_h}\sum_{a}{\xi}^{\pi,\sigma}_{M}(s,a)\left( (P(s,a)-P'(s,a) + \sigma(s,a)-\widetilde{\sigma}(s,a))^\top \mathcal{V}^{\pi,\widetilde{\sigma}}_{M',h+1}\right);\\
		\mathcal{V}^{\pi,\sigma}_{M}(s_0)-\mathcal{V}^{\pi,\sigma}_{M'}(s_0) = & \sum_{h=1}^H\sum_{s\in \mathcal{S}_h}\sum_{a}{\xi}^{\pi,\widetilde{\sigma}}_{M'}(s,a)\left( (P(s,a)-P'(s,a) + \widetilde{\sigma}(s,a) - \sigma(s,a))^\top \mathcal{V}^{\pi,\sigma}_{M}\right).
		\end{align}
	\end{lemma}
	
	\begin{proof}
		We have that
		\begin{align*}
		& \mathcal{V}^{\pi,\sigma}_{M}(s_0)-\mathcal{V}^{\pi,\widetilde{\sigma}}_{M'}(s_0) \\
		= & \mathcal{Q}^{\pi,\sigma}_{M}(s_0,\pi(s_0)) - \mathcal{Q}^{\pi,\widetilde{\sigma}}_{M'}(s_0,\pi(s_0))\\
		= & (P+\sigma)(s_0,\pi(s_0))^\top \mathcal{V}^{\pi,\sigma}_{M,2} -(P'+\widetilde{\sigma})(s_0,\pi(s_0))^\top \mathcal{V}^{\pi,\widetilde{\sigma}}_{M',2}\\
		= & (P + \sigma)(s_0,\pi(s_0))^\top (\mathcal{V}^{\pi,\sigma}_{M,2} - {\mathcal{V}}^{\pi,\widetilde{\sigma}}_{M',2}) + (P-P' + \sigma-\widetilde{\sigma})(s_0,\pi(s_0))^\top \mathcal{V}^{\pi,\widetilde{\sigma}}_{M',2}\\
		= & \sum_{h=1}^H\sum_{s\in \mathcal{S}_h}\sum_{a}{\xi}^{\pi,\sigma}_{M}(s,a)\left( (P(s,a)-P'(s,a) + \sigma(s,a)-\widetilde{\sigma}(s,a))^\top \mathcal{V}^{\pi,\widetilde{\sigma}}_{M',h+1}\right).
		\end{align*}
		
		Similarly, we can prove that
		\begin{align*}
		& \mathcal{V}^{\pi,\sigma}_{M}(s_0)-\mathcal{V}^{\pi,\widetilde{\sigma}}_{M'}(s_0) \\
		= & \mathcal{Q}^{\pi,\sigma}_{M}(s_0,\pi(s_0)) - \mathcal{Q}^{\pi,\widetilde{\sigma}}_{M'}(s_0,\pi(s_0))\\
		= & (P+\sigma)(s_0,\pi(s_0))^\top \mathcal{V}^{\pi,\sigma}_{M,2} -(P'+\widetilde{\sigma})(s_0,\pi(s_0))^\top \mathcal{V}^{\pi,\widetilde{\sigma}}_{M',2}\\
		= & (P' + \widetilde{\sigma})(s_0,\pi(s_0))^\top (\mathcal{V}^{\pi,\sigma}_{M,2} - {\mathcal{V}}^{\pi,\widetilde{\sigma}}_{M',2}) + (P' - P + \widetilde{\sigma} - \sigma)(s_0,\pi(s_0))^\top \mathcal{V}^{\pi,\sigma}_{M,2}\\
		= & \sum_{h=1}^H\sum_{s\in \mathcal{S}_h}\sum_{a}{\xi}^{\pi,\widetilde{\sigma}}_{M'}(s,a)\left( (P(s,a)-P'(s,a) + \widetilde{\sigma}(s,a) - \sigma(s,a))^\top \mathcal{V}^{\pi,\sigma}_{M}\right).
		\end{align*}
		Then we finish the proof.
	\end{proof}
	
	\section{Proofs for theorems}
	
	In this section, we give complete proofs for theorems.
	
	\subsection{Proof for Theorem~\ref{thm:main_thm}}
	\label{app:main_thm}
	
	First we define some new notations. We denote $\sigma^*$ and $\sigma'$ to be the best responses of $\pi^*$ and $\widehat{\pi}^*$ from $U(P^s)$:
	\begin{align}
	\sigma^* = & \mathop{\arg\min}\limits_{\sigma\in U{C}(P^s)} \mathcal{V}^{\pi^*,\sigma}_{M^s}(s_0),\\
	\sigma' = & \mathop{\arg\min}\limits_{\sigma\in U(P^s)} \mathcal{V}^{\widehat{\pi}^*,\sigma}_{M^s}(s_0).
	\end{align}
	
	Further we define two policy $\widetilde{\sigma}^*$ and $\widetilde{\sigma}'$ which satisfy
	\begin{align}
	\widetilde{\sigma}^* = & \mathop{\arg\min}\limits_{\sigma\in U(P^s)}\left(\max\limits_{(s,a)}||\widehat{\sigma}^* -\sigma||_1\right), \\
	\widetilde{\sigma}' = & \mathop{\arg\min}\limits_{\sigma\in U(\widehat{P})}\left(\max\limits_{(s,a)}||\sigma' - \sigma||_1\right).
	\end{align}
	
	Then we decompose the error term $Err(\widehat{\pi})$ with Lemma.~\ref{lem:err_decompose} (which corresponds to Lemma.~\ref{lem:err_upper_bound} above).
	\begin{lemma}
		\label{lem:err_decompose}
		The error function $Err$ over policy $\widehat{\pi}^*$ can be bounded by
		\begin{equation}
		Err(\widehat{\pi}^*)\leq \left( \mathcal{V}^{\pi^*,\widetilde{\sigma}^*}_{M^s}(s_0) - \mathcal{V}^{{\pi}^*,\widehat{\sigma}^*}_{\widehat{M}}(s_0) \right)  + \left( \mathcal{V}^{\widehat{\pi}^*,\widetilde{\sigma}'}_{\widehat{M}}(s_0) -  \mathcal{V}^{\widehat{\pi}^*,{\sigma}'}_{M^s}(s_0) \right) .
		\end{equation}
	\end{lemma}
	The proof for this lemma is given in Appendix~\ref{app:proof_err_decompose}.
	
	\subsubsection{Step 1:} 
	\label{app:e1}
	For the first term $\mathcal{V}^{\pi^*,\widetilde{\sigma}^*}_{M^s}(s_0) - \mathcal{V}^{{\pi}^*,\widehat{\sigma}^*}_{\widehat{M}}(s_0)$, we apply Lemma~\ref{lem:v_gap_decompose_gen} to get
	\begin{align}
	& \mathcal{V}^{\pi^*,\widetilde{\sigma}^*}_{M^s}(s_0) - \mathcal{V}^{{\pi}^*,\widehat{\sigma}^*}_{\widehat{M}}(s_0) \notag\\
	= & \sum_{h=1}^H\sum_{s\in \mathcal{S}_h}\sum_{a}{\xi}^{\pi^*,\widetilde{\sigma}^*}_{M^s}(s,a)\left( (P^s(s,a)-\widehat{P}(s,a) + \widetilde{\sigma}^*(s,a)-\widehat{\sigma}^*(s,a))^\top \mathcal{V}^{{\pi}^*,\widehat{\sigma}^*}_{\widehat{M},h+1}\right)\notag\\
	= & \underbrace{\sum_{h=1}^H\sum_{s\in \mathcal{S}_h}\sum_{a}{\xi}^{\pi^*,\widetilde{\sigma}^*}_{M^s}(s,a)\left( (P^s(s,a)-\widehat{P}(s,a))^\top \mathcal{V}^{{\pi}^*,\widehat{\sigma}^*}_{\widehat{M},h+1}\right)}_{e_{11}} + \underbrace{\sum_{h=1}^H\sum_{s\in \mathcal{S}_h}\sum_{a}{\xi}^{\pi^*,\widetilde{\sigma}^*}_{M^s}(s,a)\left( (\widetilde{\sigma}^*(s,a)-\widehat{\sigma}^*(s,a))^\top \mathcal{V}^{{\pi}^*,\widehat{\sigma}^*}_{\widehat{M},h+1}\right)}_{e_{12}}.
	\label{eq:step_1_gen}
	\end{align}
	
	Recall our definition of $\alpha(N)$ in Eq.~\ref{eq:alpha_def}:
		\begin{equation}
	    \alpha(N)=\sqrt{2D\ln(2/\delta')/N}.
	\end{equation}
	For the first part $e_{11}$, on the good event $\mathcal{E}_G$, we have below inequalities:
	\begin{align}
	& \sum_{h=1}^H\sum_{s\in \mathcal{S}_h}\sum_{a}{\xi}^{\pi^*,\widetilde{\sigma}^*}_{M^s}(s,a)\left( (P^s(s,a)-\widehat{P}(s,a))^\top \mathcal{V}^{{\pi}^*,\widehat{\sigma}^*}_{\widehat{M},h+1}\right)\\
	\leq & \sum_{h=1}^H\sum_{s\in \mathcal{S}_h}\sum_{a}{\xi}^{\pi^*,\widetilde{\sigma}^*}_{M^s}(s,a) ||P^s(s,a)-\widehat{P}(s,a)||_1 ||\mathcal{V}^{{\pi}^*,\widehat{\sigma}^*}_{\widehat{M},h+1}||_\infty\\
	\leq & \sum_{h=1}^H\sum_{s\in \mathcal{S}_h}\sum_{a}{\xi}^{\pi^*,\widetilde{\sigma}^*}_{M^s}(s,a) H\alpha(N).
	\end{align}
	
	For the second term $e_{12}$, using the Assumption \ref{assp:pwc} and the good event $\mathcal{E}_G$, we have that
	\begin{align}
	& \sum_{h=1}^H\sum_{s\in \mathcal{S}_h}\sum_{a}{\xi}^{\pi^*,\widetilde{\sigma}^*}_{M^s}(s,a)\left(( \widetilde{\sigma}^*(s,a)-\widehat{\sigma}^*(s,a))^\top \mathcal{V}^{{\pi}^*,\widehat{\sigma}^*}_{\widehat{M},h+1}\right) \\
	\leq & \sum_{h=1}^H\sum_{s\in \mathcal{S}_h}\sum_{a}{\xi}^{\pi^*,\widetilde{\sigma}^*}_{M^s}(s,a) ||\widetilde{\sigma}^*(s,a)-\widehat{\sigma}^*(s,a)||_1 ||\mathcal{V}^{{\pi}^*,\widehat{\sigma}^*}_{\widehat{M},h+1}||_\infty \\
	\leq & \sum_{h=1}^H\sum_{s\in \mathcal{S}_h}\sum_{a}{\xi}^{\pi^*,\widetilde{\sigma}^*}_{M^s}(s,a) \lambda H\alpha(N).	\label{eq:e_12_bound}
	\end{align}
	
	Now combining above two parts we get that on the good event $\mathcal{E}_G$,
	\begin{align}
	\mathcal{V}^{\pi^*,\widetilde{\sigma}^*}_{M^s}(s_0) - \mathcal{V}^{{\pi}^*,\widehat{\sigma}^*}_{\widehat{M}}(s_0)\leq & \sum_{h=1}^H\sum_{s\in \mathcal{S}_h}\sum_{a}{\xi}^{\pi^*,\widetilde{\sigma}^*}_{M^s}(s,a) (1+\lambda)\alpha(N)\\
	= & (1+\lambda)H^2 \alpha(N).
	\end{align}
	
	\subsubsection{Step 2:} 
	\label{app:e2}
	For the first term $\mathcal{V}^{\widehat{\pi}^*,\widetilde{\sigma}'}_{\widehat{M}}(s_0) -  \mathcal{V}^{\widehat{\pi}^*,{\sigma}'}_{M^s}(s_0)$, we apply Lemma~\ref{lem:v_gap_decompose_gen} to get
	\begin{align}
	& \mathcal{V}^{\widehat{\pi}^*,\widetilde{\sigma}'}_{\widehat{M}}(s_0) -  \mathcal{V}^{\widehat{\pi}^*,{\sigma}'}_{M^s}(s_0) \notag\\
	= & \sum_{h=1}^H\sum_{s\in \mathcal{S}_h}\sum_{a}{\xi}^{\widehat{\pi}^*,{\sigma}'}_{M^s}(s,a)\left( (\widehat{P}(s,a) - P^s(s,a) + \widehat{\sigma}^*(s,a) - \widetilde{\sigma}^*(s,a))^\top \mathcal{V}^{\widehat{\pi}^*,\widetilde{\sigma}'}_{\widehat{M},h+1}\right)\notag\\
	= & \underbrace{\sum_{h=1}^H\sum_{s\in \mathcal{S}_h}\sum_{a}{\xi}^{\widehat{\pi}^*,{\sigma}'}_{M^s}(s,a)\left( (\widehat{P}(s,a) - P^s(s,a))^\top \mathcal{V}^{\widehat{\pi}^*,\widetilde{\sigma}'}_{\widehat{M},h+1}\right)}_{e_{21}} + \underbrace{\sum_{h=1}^H\sum_{s\in \mathcal{S}_h}\sum_{a}{\xi}^{\widehat{\pi}^*,{\sigma}'}_{M^s}(s,a)\left( (\widetilde{\sigma}'(s,a) - {\sigma}'(s,a))^\top \mathcal{V}^{\widehat{\pi}^*,\widetilde{\sigma}'}_{\widehat{M},h+1}\right)}_{e_{22}}.
	\label{eq:step_2_gen}
	\end{align}
	
	For the first part $e_{21}$, on the good event $\mathcal{E}_G$, we have below inequalities according to the definition of Eq.~\eqref{eq:alpha_def}:
	\begin{align}
	\sum_{h=1}^H\sum_{s\in \mathcal{S}_h}\sum_{a}{\xi}^{\widehat{\pi}^*,{\sigma}'}_{M^s}(s,a)\left( (\widehat{P}(s,a) - P^s(s,a))^\top \mathcal{V}^{\widehat{\pi}^*,\widetilde{\sigma}'}_{\widehat{M},h+1}\right) \leq & \sum_{h=1}^H\sum_{s\in \mathcal{S}_h}\sum_{a}{\xi}^{\widehat{\pi}^*,{\sigma}'}_{M^s}(s,a) H \alpha(N).
	\end{align}
	
	For the second term $e_{22}$, we have that
	\begin{align}
	& \sum_{h=1}^H\sum_{s\in \mathcal{S}_h}\sum_{a}{\xi}^{\widehat{\pi}^*,{\sigma}'}_{M^s}(s,a)\left( (\widetilde{\sigma}'(s,a) - {\sigma}'(s,a))^\top \mathcal{V}^{\widehat{\pi}^*,\widetilde{\sigma}'}_{\widehat{M},h+1}\right) \\
	\leq & \sum_{h=1}^H\sum_{s\in \mathcal{S}_h}\sum_{a}{\xi}^{\widehat{\pi}^*,{\sigma}'}_{M^s}(s,a) ||\widetilde{\sigma}'(s,a) - {\sigma}'(s,a)||_1 ||\mathcal{V}^{\widehat{\pi}^*,\widetilde{\sigma}'}_{\widehat{M},h+1}||_\infty \\
	\leq & \sum_{h=1}^H\sum_{s\in \mathcal{S}_h}\sum_{a}{\xi}^{\widehat{\pi}^*,{\sigma}'}_{M^s}(s,a)H\lambda \alpha(N). \label{eq:e_22_bound}
	\end{align}
	
	Now combining above two parts we get that
	\begin{align}
	\mathcal{V}^{\widehat{\pi}^*,\widetilde{\sigma}'}_{\widehat{M}}(s_0) -  \mathcal{V}^{\widehat{\pi}^*,{\sigma}'}_{M^s}(s_0)\leq & \sum_{h=1}^H\sum_{s\in \mathcal{S}_h}\sum_{a}{\xi}^{\widehat{\pi}^*,{\sigma}'}_{M^s}(s,a) (1+\lambda )H\alpha(N) \\
	= & (1+\lambda)H^2 \alpha(N).
	\end{align}
	
	\subsubsection{Step 3}
	We sum the results of above two steps to get
	\begin{align}
	Err(\widehat{\pi}^*) \leq & 2(1+\lambda)H^2 \alpha(N).
	\end{align}
	
	If we have that
	\begin{equation}
	    N\geq \alpha^{-1}(\epsilon/(2(1+\lambda)H^2))=\frac{8(1+\lambda)^2H^4D\ln(2/\delta')}{\epsilon^2},
	\end{equation}
	Then on the good event $\mathcal{E}_G$, we have that
	\begin{align}
	Err(\widehat{\pi}^*)\leq \epsilon.
	\end{align}
	
    \subsection{Proof for Theorem~\ref{thm:pwc_thm}}
    \label{app:pwc_thm}
    
    Here we consider theorem \ref{thm:pwc_thm} for the PWC cases, where the constraint function $\mathcal{C}$ satisfies Assumption~\ref{assp:pwc}. 
	
	The key that Assumption~\ref{assp:pwc} can improve the sample complexity is the independence of $\widehat{\sigma}(s,a)$ over all $(s,a)$ pairs. This independence is guaranteed because we can calculate each $\widehat{\sigma}(s,a)$ separately. If $\mathcal{C}$ is not a PWC, $\widehat{\sigma}(s,a)$ can be correlated over $(s,a)$ pairs. 
	
	For state-action pair $(s,a)$ where $s\in\mathcal{S}_h$, the random variable $\widehat{P}(s,a)$ is independent of any $\mathcal{V}^{\pi,\sigma}_{h+1}$ where $(\pi,\sigma)$ is calculated without using $\widehat{P}(s,a)$. Then we can apply Lemma~\ref{lem:bernstein}. Thus for PWC, we have that with a probability at least $1-\delta'$,
	\begin{align}
	(P^s(s,a)-\widehat{P}(s,a))^\top \mathcal{V}^{{\pi}^*,\widehat{\sigma}^*}_{\widehat{M},h+1} \leq & \sqrt{\frac{2Var_{P^s(s,a)}({\mathcal{V}}^{{\pi}^*,\widehat{\sigma}^*}_{\widehat{M},h+1})\ln(2/\delta')}{N}} + \frac{H\ln(2/\delta')}{3N}
	\leq \beta(N),\label{eq:beta_e_11}\\
	(\widehat{P}(s,a)-P^s(s,a))^\top \mathcal{V}^{\widehat{\pi}^*,\widetilde{\sigma}'}_{\widehat{M},h+1} \leq & \sqrt{\frac{2Var_{P^s(s,a)}(\mathcal{V}^{\widehat{\pi}^*,\widetilde{\sigma}'}_{\widehat{M},h+1})\ln(2/\delta')}{N}} + \frac{H\ln(2/\delta')}{3N} \leq \beta(N),\label{eq:beta_e_21}
 	\end{align}
	where we define that
	\begin{align}
	\beta(N)=H\sqrt{\frac{2\ln(2/\delta')}{N}} + \frac{H\ln(2/\delta')}{3N}.
	\end{align}
	
	Recall the proof process for Theorem \ref{thm:main_thm}, we can upper bound $Err(\widehat{\pi}^*)$ using Lemma~\ref{lem:err_decompose} and Eq.~\eqref{eq:step_1_gen} \& \eqref{eq:step_2_gen}:
	\begin{equation}
	    Err(\widehat{\pi}^*)\leq e_{11} + e_{12} + e_{21} + e_{22}.
	\end{equation}
	
	With the independence property of PWC in Eq.~\eqref{eq:beta_e_11} and \eqref{eq:beta_e_21}, we have that
	\begin{align}
	    e_{11} \leq &  \sum_{h=1}^H\sum_{s\in \mathcal{S}_h}\sum_{a}{\xi}^{\pi^*,\widetilde{\sigma}^*}_{M^s}(s,a)
	    \beta(N)\leq H\beta(N),\label{eq:pwc_e_11_bound}\\
	    e_{21}\leq &  \sum_{h=1}^H\sum_{s\in \mathcal{S}_h}\sum_{a}{\xi}^{\widehat{\pi}^*,{\sigma}'}_{M^s}(s,a)
	    \beta(N)\leq H\beta(N).\label{eq:pwc_e_21_bound}
	\end{align}
 	
	Combining above inequalities and Eq.~\eqref{eq:e_12_bound} \& \eqref{eq:e_22_bound}, we get
	\begin{equation}
	    Err(\widehat{\pi}^*)\leq 2H(\beta(N)+\lambda H\alpha(N)).
	\end{equation}
	
	Now if we have
	\begin{equation}
	    N\geq \frac{\ln(2/\delta')}{18},
	\end{equation}
	we have that
	\begin{align}
	    H\sqrt{\frac{2\ln(2/\delta')}{N}} \geq & \frac{H\ln(2/\delta')}{3N},\\
	    \beta(N)\leq & 2H\sqrt{\frac{2\ln(2/\delta')}{N}}.
	\end{align}
	
	Now if we have that
	\begin{equation}
	N\geq \max\left(\frac{\ln(2/\delta')}{18}, \frac{8H^4(2+\lambda\sqrt{D})^2\ln(2/\delta')}{\epsilon'^2} \right) = \frac{8H^4(2+\lambda\sqrt{D})^2\ln(2/\delta')}{\epsilon'^2},
	\end{equation}
	then by setting $\delta'=\delta/(4SA)$ and using the union bound, we have that with a probability at least $1-\delta$,
	\begin{equation}
	Err(\widehat{\pi}^*)\leq \epsilon.
	\end{equation}
	
	Then we finish the proof for theorem~\ref{thm:pwc_thm}.
	
	Specifically, for the fixed perturbation constraint in Sec.~\ref{sec:fix_perturbations}
    
    \subsection{Proof for Theorem~\ref{thm:tvdc_thm}}
    \label{app:tvdc_thm}
    
    Here we modify the results of Theorem~\ref{thm:tvdc_thm} in the paper. The constants are modified. We give a modified version here.
    
	\begin{theorem}[Modified version for Theorem~\ref{thm:tvdc_thm}]
		\label{thm:tvdc_thm_modified}
		Consider a given $(\epsilon, \delta)$, where $\epsilon\in(0,H)$ and $\delta\in(0,1)$, and a TVDC $\mathcal{C}_{TVD}(\cdot;u)$. If $u\leq \epsilon/(16H^2)$, then using
		\begin{equation}
		    N\geq \frac{128H^4\ln(8SA/\delta)}{\epsilon^2},
		\end{equation}
		with a probability at least $1-\delta$ we have 
		\begin{equation}
		    Err(\widehat{\pi}^*) \leq \epsilon.
		\end{equation}
		If $u>\epsilon/(16H^2)$, then using
		\begin{small}
		\begin{align}
		N\geq&  \max\left(\frac{16H^4\ln(8SA/\delta)}{\epsilon^2}(\sqrt{2}+6\sqrt{uD})^2,\frac{49D\ln(8SA/\delta)}{27u}\right),
		\end{align}
		\end{small}
		then with a probability no less than $1-\delta$, $\widehat{\pi}^*$ satisfies
		\begin{align}
		Err(\widehat{\pi}^*) \leq \epsilon.
		\end{align}
	\end{theorem}
	
	We give complete proof below.
	
	Here we consider a specific PWC constraint, the Total Variation Distance Constraint (TVDC). Formally, for some small value $u>0$ and any probability distribution $p\in\Delta_D$, the TVDC is define as
	\begin{equation}
		\mathcal{C}_{TVD}(p):=\{p'\in\Delta_D:||p'-p||_{TV}\leq u\}.
	\end{equation}
	
	Since we only consider the Categorical distribution $p,p'\in\Delta_D$, thus we have that
	\begin{align}
		||p'-p||_{TV}= & \frac{1}{2}||p'-p||_1\\
		= & \frac{1}{2}\sum_{i=1}^D|p'(i)-p(i)|.
	\end{align}
	
	Now we consider state-action pair $(s,a)$ and its corresponding simulator transition $P^s(s,a)$ and estimated transition $\widehat{P}(s,a)$. 

    \subsubsection{Case 1}
    If we have that $u\leq \epsilon/(16H^2)$, then we consider $\sigma\in U(P^s(s,a))$ and $U(\widehat{P}(s,a))$ and we have
    \begin{equation}
        ||\widetilde{\sigma}^*(s,a)-\widehat{\sigma}^*(s,a)||_1\leq ||\widetilde{\sigma}^*(s,a)||_1 + ||\widehat{\sigma}^*(s,a)||_1\leq 4u\leq \frac{\epsilon}{4H^2}.
    \end{equation}
    
    Then we would have that
    \begin{align}
        e_{12} \leq &  \sum_{h=1}^H\sum_{s\in \mathcal{S}_h}\sum_{a}{\xi}^{\pi^*,\widetilde{\sigma}^*}_{M^s}(s,a)
	    \frac{\epsilon}{4H}\leq \frac{\epsilon}{4},\label{eq:tvdc_e_12_bound}\\
	    e_{22}\leq &  \sum_{h=1}^H\sum_{s\in \mathcal{S}_h}\sum_{a}{\xi}^{\widehat{\pi}^*,{\sigma}'}_{M^s}(s,a)
	    \frac{\epsilon}{4H}\leq \frac{\epsilon}{4}.\label{eq:tvdc_e_22_bound}
    \end{align}
    
    Therefore we can upper bound $Err(\widehat{\pi}^*)$ with Lemma~\ref{lem:err_decompose}, Eq.~\eqref{eq:step_1_gen}, \eqref{eq:step_2_gen}, \eqref{eq:pwc_e_11_bound} and \eqref{eq:pwc_e_21_bound}:
    \begin{equation}
        Err(\widehat{\pi}^*)\leq e_{11} + e_{12} + e_{21} + e_{22} \leq 2H\beta(N) + \frac{\epsilon}{2}.
    \end{equation}
    
    Thus if we have 
    \begin{equation}
        N\geq \frac{128H^4\ln(2/\delta')}{\epsilon^2},
    \end{equation}
    then we have that
    \begin{align}
        \beta(N)\leq 2H\sqrt{\frac{2\ln(2/\delta')}{N}} \leq \frac{\epsilon}{4H}.
    \end{align}
    Finally by letting $\delta'=\delta/(4SA)$ and using the union bound, we have that with a probability at least $1-\delta$,
    \begin{equation}
        Err(\widehat{\pi}^*)\leq \epsilon.
    \end{equation}
	
	\subsubsection{Case 2}
	Now we consider the case that $u>\epsilon/(16H^2)$.
	
	Firstly, we upper bound $Err(\widehat{\pi}^*)$ with Lemma~\ref{lem:err_decompose}, Eq.~\eqref{eq:step_1_gen}, \eqref{eq:step_2_gen}, \eqref{eq:pwc_e_11_bound} and \eqref{eq:pwc_e_21_bound}:
	\begin{equation}
	    Err(\widehat{\pi}^*)\leq 2H\beta(N) + e_{12} + e_{22}.
	    \label{eq:tvdc_case2_error_bound}
	\end{equation}
	Then we upper bound $e_{12}$ and $e_{22}$ separately.
	
	\textbf{Upper bounding $e_{12}$.}
	
	For each state-action pair $(s,a)$, we aim to bound the term
	\begin{equation}
		||\widetilde{\sigma}^*(s,a)-\widehat{\sigma}^*(s,a)||_1 ||\mathcal{V}^{{\pi}^*,\widehat{\sigma}^*}_{\widehat{M},h+1}||_\infty,
	\end{equation}
	where
	\begin{align}
		\widehat{\sigma}^*(s,a) = & \mathop{\arg\min}\limits_{\ell\in U(\widehat{P})(s,a)} (\widehat{P}(s,a)+\ell)^\top \mathcal{V}^{\widehat{\pi}^*,\widehat{\sigma}^*}_{\widehat{M},h+1},\\
		\widetilde{\sigma}^*(s,a) = & \mathop{\arg\min}\limits_{\ell\in U(P^s)(s,a)}||\widehat{\sigma}^* -\ell||.
	\end{align}
	
	We have that
	\begin{equation}
		U(\widehat{P})(s,a) = \{p-\widehat{P}(s,a):p\in\Delta_{D_{sa}},\ ||p-\widehat{P}(s,a)||_1\leq 2u\}.
	\end{equation}
	
	For convenience we denote $\ell^*:=\widehat{\sigma}^*(s,a)$ and $\widetilde{\ell}:=\widetilde{\sigma}^*(s,a)$. Since we constraint $\widehat{P}+\ell^*$ to be a vector in $\Delta_{D_{sa}}$, we have that
	\begin{equation}
		\sum\limits_{i=1}^{{D_{sa}}} \ell^*(i) = 0.
	\end{equation}
	
	Further, recall the definition of TVDC, we have that 
	\begin{equation}
	    ||\ell^*||_1\leq 2u, ||\widetilde{\ell}||_1\leq 2u.
	\end{equation}
	
	Now, we turn to find the solution for $\widetilde{\ell}$. First we define a vector ${\ell}$. For each $i\in [{D_{sa}}]$, we consider $\ell^*(i)$ for five cases. 
	
	\begin{itemize}
		\item If $\ell^*(i)=0$, we can simply set ${\ell}(i)=0$.
		
		\item If $\ell^*(i)>0$ and ${P}^s(i;s,a)+\ell^*(i)\leq 1$, we can set ${\ell}(i)=\ell^*(i)$.
		
		\item If $\ell^*(i)>0$ and ${P}^s(i;s,a)+\ell^*(i)> 1$, we set ${\ell}(i) = 1-{P}^s(i;s,a)$.
		
		\item If $\ell^*(i)<0$ and ${P}^s(i;s,a)+\ell^*(i)\geq 0$, we can set ${\ell}(i)=\ell^*(i)$.
		
		\item If $\ell^*(i)<0$ and ${P}^s(i;s,a)+\ell^*(i)< 0$, we set ${\ell}(i) = -{P}^s(i;s,a)$.
	\end{itemize}

	\begin{lemma}
	\label{lem:e12_ell}
		With $\ell$ defined above, we have
		\begin{equation}
			||\ell^* -\widetilde{\ell}||_1 \leq 2||\ell^* -\ell||_1.
			\label{eq:e12_ell}
		\end{equation}
	\end{lemma}
	We give the proof of this lemma in Appendix~\ref{app:proof_e12_ell}.

	Then we can turn to bound $||\ell^* -{\ell}||_1$.
	
	According to our design of $\ell$, we know that for any $i\in[D]$, we always have $\ell^*(i)\ell(i)\geq 0$.
	
	Firstly, we define the sets
	\begin{align}
	    L^+:=&\{i\in[D]:\ell^*(i)>0,\ P^s(i;s,a)+\ell^*(i)>1\},\\
	    L^-:=&\{i\in[D]:\ell^*(i)<0,\ P^s(i;s,a)+\ell^*(i)< 0\}.
	\end{align}
	It is easy to see that only elements in $L^+$ and $L^-$ that will cause the difference between $\ell^*$ and ${\ell}$.
	
	Consider $L^+$ first. For $i\in L^+$, it is easy to see that
	\begin{align}
	    \ell^*(i)-{\ell}(i) = & \ell^*(i) + {P}^s(i;s,a) - 1 \\
	    \leq & \ell^*(i) + P^s(i;s,a) - \ell^*(i) - \widehat{P}(i;s,a)\\
	    = &  P^s(i;s,a) -\widehat{P}(i;s,a)\\
	    \leq & \sqrt{\frac{2P^s(i;s,a)(1-P^s(i;s,a))\ln(2/\delta')}{N}} + \frac{\ln(2/\delta')}{3N}.
	\end{align}
	The last inequality holds due to Lemma~\ref{lem:bernstein}.
	
	Notice that 
	\begin{align}
	\label{eq:sum_1_minus_p}
	    \sum\limits_{i\in L^+} (1-P^s(i;s,a))\leq \sum\limits_{i\in L^+} \ell^*(i) \leq u.
	\end{align}
	
	Therefore we have that
	\begin{align}
	    & \sum\limits_{i\in L^+}\sqrt{\frac{2P^s(i;s,a)(1-P^s(i;s,a))\ln(2/\delta')}{N}}\\ =& (\sum\limits_{i'\in L^+}P^s(i';s,a))\sum\limits_{i\in L^+}\frac{P^s(i;s,a)}{\sum\limits_{i''\in L^+}P^s(i'';s,a)}\sqrt{\frac{2(1-P^s(i;s,a))\ln(2/\delta')}{P^s(i;s,a)N}}\\
	    \leq & (\sum\limits_{i'\in L^+}P^s(i';s,a))\sqrt{\sum\limits_{i\in L^+}\frac{2(1-P^s(i;s,a))\ln(2/\delta')}{\sum\limits_{i''\in L^+}P^s(i'';s,a)N}} \\
	    = & \sqrt{(\sum\limits_{i'\in L^+}P^s(i';s,a))\frac{\sum\limits_{i\in L^+}2(1-P^s(i;s,a))\ln(2/\delta')}{N}} \\
	    \leq & \sqrt{\frac{2u\ln(2/\delta')}{N}}.
	\end{align}
	The first inequality holds with the Cauchy–Schwarz inequality, and the second holds due to inequality~\eqref{eq:sum_1_minus_p}.
	
	Hence we have that
	\begin{align}
	    \sum\limits_{i\in L^+} |\ell^*(i)-\ell(i)| \leq \sqrt{\frac{2u\ln(2/\delta')}{N}} + \frac{D\ln(2/\delta')}{3N}.
	\end{align}
	
	Now we turn to consider $L^-$. For $i\in L^-$, using that we have $\ell^*(i) \leq -P^s(i;s,a)$, we have
	\begin{align}
	    \ell(i) - \ell^*(i) \leq & - {P}^s(i;s,a) - (-\widehat{P}(i;s,a))\\
	    = &  \widehat{P}(i;s,a) - P^s(i;s,a)\\
	    \leq & \sqrt{\frac{2P^s(i;s,a)(1-P^s(i;s,a))\ln(2/\delta')}{N}} + \frac{\ln(2/\delta')}{3N}.
	\end{align}
	
	Recall that 
	\begin{align}
	\label{eq:sum_p}
	    \sum\limits_{i\in L^-} P^s(i;s,a) \leq \sum\limits_{i\in L^-}-\ell^*(i) \leq u.
	\end{align}
	
	Hence we have that
	\begin{align}
	    & \sum\limits_{i\in L^-}\sqrt{\frac{2P^s(i;s,a)(1-P^s(i;s,a))\ln(2/\delta')}{N}}\\ =& (\sum\limits_{i'\in L^-}P^s(i';s,a))\sum\limits_{i\in L^-}\frac{P^s(i;s,a)}{\sum\limits_{i''\in L^-}P^s(i'';s,a)}\sqrt{\frac{2(1-P^s(i;s,a))\ln(2/\delta')}{P^s(i;s,a)N}}\\
	    \leq & (\sum\limits_{i'\in L^-}P^s(i';s,a))\sqrt{\sum\limits_{i\in L^+}\frac{2\ln(2/\delta')}{\sum\limits_{i''\in L^-}P^s(i'';s,a)N}} \\
	    = & \sqrt{(\sum\limits_{i'\in L^-}P^s(i';s,a))\frac{\sum\limits_{i\in L^-}2\ln(2/\delta')}{N}} \\
	    \leq & \sqrt{\frac{2uD\ln(2/\delta')}{N}}.
	\end{align}
	The first inequality holds with the Cauchy–Schwarz inequality, and the second holds due to inequality~\eqref{eq:sum_p}.
	
	Now we have that
    \begin{align}
        ||\widetilde{\sigma}^*(s,a)-\widehat{\sigma}^*(s,a)||_1 \leq 2\left(\sqrt{\frac{2u\ln(2/\delta')}{N}}+\sqrt{\frac{2uD\ln(2/\delta')}{N}}+\frac{D\ln(2/\delta')}{3N}\right).
    \end{align}	
    
    Now if we have
    \begin{equation}
        N\geq \frac{D\ln(2/\delta')}{18u},
    \end{equation}
    we have that 
    \begin{equation}
        \sqrt{\frac{2uD\ln(2/\delta')}{N}}\geq\frac{D\ln(2/\delta')}{3N}.
    \end{equation}
    
    Then we have 
    \begin{equation}
        ||\widetilde{\sigma}^*(s,a)-\widehat{\sigma}^*(s,a)||_1 \leq 6\sqrt{\frac{2uD\ln(2/\delta')}{N}}.
    \end{equation}
	
	Thus we have that: 
	\begin{align}
	\label{eq:tvdc_case2_e12}
	    e_{12}\leq 6H^2\sqrt{\frac{2uD\ln(2/\delta')}{N}}.
	\end{align}
	
	\textbf{Upper bounding $e_{22}$}
	We use similar techniques to that for $e_{12}$ to upper bound $e_{22}$. We give all the details below.

    For each state-action pair $(s,a)$, we aim to bound the term
	\begin{equation}
		||\widetilde{\sigma}'(s,a)-{\sigma}'(s,a)||_1 ||\mathcal{V}^{\widehat{\pi}^*,\widetilde{\sigma}'}_{\widehat{M},h+1}||_\infty,
	\end{equation}
	where
	\begin{align}
		{\sigma}'(s,a) = & \mathop{\arg\min}\limits_{\ell\in U({P}^s)(s,a)} ({P}^s(s,a)+\ell)^\top \mathcal{V}^{\widehat{\pi}^*,{\sigma}'}_{{M}^s,h+1},\\
		\widetilde{\sigma}'(s,a) = & \mathop{\arg\min}\limits_{\ell\in U(\widehat{P})(s,a)}||{\sigma}' -\ell||.
	\end{align}
	
	For convenience we denote $\ell':={\sigma}'(s,a)$ and $\widetilde{\ell}:=\widetilde{\sigma}'(s,a)$. Since we constraint $P^s+\ell'$ to be a vector in $\Delta_{D_{sa}}$, we have that
	\begin{equation}
		\sum\limits_{i=1}^{{D_{sa}}} \ell'(i) = 0.
	\end{equation}
	
	Further, recall the definition of TVDC, we have that 
	\begin{equation}
	    ||\ell'||_1\leq 2u, ||\widetilde{\ell}||_1\leq 2u.
	\end{equation}
	
	Now, we turn to find the solution for $\widetilde{\ell}$. First we define a vector ${\ell}$. For each $i\in [{D_{sa}}]$, we consider $\ell'(i)$ for five cases. 
	
	\begin{itemize}
		\item If $\ell'(i)=0$, we can simply set ${\ell}(i)=0$.
		
		\item If $\ell'(i)>0$ and $\widehat{P}(i;s,a)+\ell'(i)\leq 1$, we can set ${\ell}(i)=\ell'(i)$.
		
		\item If $\ell'(i)>0$ and $\widehat{P}(i;s,a)+\ell'(i)> 1$, we set ${\ell}(i) = 1-\widehat{P}(i;s,a)$.
		
		\item If $\ell'(i)<0$ and $\widehat{P}(i;s,a)+\ell'(i)\geq 0$, we can set ${\ell}(i)=\ell'(i)$.
		
		\item If $\ell'(i)<0$ and $\widehat{P}(i;s,a)+\ell'(i)< 0$, we set ${\ell}(i) = -\widehat{P}(i;s,a)$.
	\end{itemize}

	\begin{lemma}
	\label{lem:e22_ell}
		With $\ell$ defined above, we have
		\begin{equation}
			||\ell' -\widetilde{\ell}||_1 \leq 2||\ell' -\ell||_1.
			\label{eq:e22_ell}
		\end{equation}
	\end{lemma}
	The proof process for Lemma~\ref{lem:e22_ell} is the same as that for Lemma~\ref{lem:e12_ell} by simply replacing the parameters correspondingly. For the completeness of result, we give the proof in Appendix~\ref{app:proof_e22_ell}.

	Then we can turn to bound $||\ell' -{\ell}||_1$.
	
	According to our design of $\ell$, we know that for any $i\in[D]$, we always have $\ell'(i)\ell(i)\geq 0$.
	
	Firstly, we define the sets
	\begin{align}
	    L^+:=&\{i\in[D]:\ell'(i)>0,\ \widehat{P}(i;s,a)+\ell'(i)>1\},\\
	    L^-:=&\{i\in[D]:\ell'(i)<0,\ \widehat{P}(i;s,a)+\ell'(i)< 0\}.
	\end{align}
	It is easy to see that only elements in $L^+$ and $L^-$ that will cause the difference between $\ell'$ and ${\ell}$.
	
	Consider $L^+$ first. For $i\in L^+$, it is easy to see that
	\begin{align}
	    \ell'(i)-{\ell}(i) = & \ell'(i) + \widehat{P}(i;s,a) - 1 \\
	    \leq & \ell’(i) + \widehat{P}(i;s,a) - \ell‘(i) - {P}^s(i;s,a)\\
	    = &  \widehat{P}(i;s,a)-P^s(i;s,a)\\
	    \leq & \sqrt{\frac{2\widehat{P}(i;s,a)(1-\widehat{P}(i;s,a))\ln(2/\delta')}{N-1}} + \frac{7\ln(2/\delta')}{3N}.
	\end{align}
	The last inequality holds due to the empirical Bernstein's inequality~\cite{maurer2009empirical}.
	
	Notice that 
	\begin{align}
	\label{eq:sum_2_minus_hat_p}
	    \sum\limits_{i\in L^+} (1-\widehat{P}(i;s,a))\leq \sum\limits_{i\in L^+} \ell^*(i) \leq u.
	\end{align}
	
	Therefore we have that
	\begin{align}
	    & \sum\limits_{i\in L^+}\sqrt{\frac{2\widehat{P}(i;s,a)(1-\widehat{P}(i;s,a))\ln(2/\delta')}{N-1}}\\
	    =& (\sum\limits_{i'\in L^+}\widehat{P}(i';s,a))\sum\limits_{i\in L^+}\frac{\widehat{P}(i;s,a)}{\sum\limits_{i''\in L^+}\widehat{P}(i'';s,a)}\sqrt{\frac{2(1-\widehat{P}(i;s,a))\ln(2/\delta')}{\widehat{P}(i;s,a)(N-1)}}\\
	    \leq & (\sum\limits_{i'\in L^+}\widehat{P}(i';s,a))\sqrt{\sum\limits_{i\in L^+}\frac{2(1-\widehat{P}(i;s,a))\ln(2/\delta')}{\sum\limits_{i''\in L^+}\widehat{P}(i'';s,a)(N-1)}} \\
	    = & \sqrt{(\sum\limits_{i'\in L^+}\widehat{P}(i';s,a))\frac{\sum\limits_{i\in L^+}2(1-\widehat{P}(i;s,a))\ln(2/\delta')}{N-1}} \\
	    \leq & \sqrt{\frac{2u\ln(2/\delta')}{N-1}}.
	\end{align}
	The first inequality holds with the Cauchy–Schwarz inequality, and the second holds due to inequality~\eqref{eq:sum_2_minus_hat_p}.
	
	Hence we have that
	\begin{align}
	    \sum\limits_{i\in L^+} |\ell'(i)-\ell(i)| \leq \sqrt{\frac{2u\ln(2/\delta')}{N-1}} + \frac{7D\ln(2/\delta')}{3N}.
	\end{align}
	
	Now we turn to consider $L^-$. For $i\in L^-$, using that we have $\ell'(i) \leq -\widehat{P}(i;s,a)$, we have
	\begin{align}
	    \ell(i) - \ell'(i) \leq & - \widehat{P}(i;s,a) - (-{P}^s(i;s,a))\\
	    = &  P^s(i;s,a) - \widehat{P}(i;s,a)\\
	    \leq & \sqrt{\frac{2\widehat{P}(i;s,a)(1-\widehat{P}(i;s,a))\ln(2/\delta')}{N-1}} + \frac{7\ln(2/\delta')}{3N}.
	\end{align}
	
	Recall that 
	\begin{align}
	\label{eq:sum_hat_p}
	    \sum\limits_{i\in L^-} \widehat{P}(i;s,a) \leq \sum\limits_{i\in L^-}-\ell^*(i) \leq u.
	\end{align}
	
	Hence we have that
	\begin{align}
	    & \sum\limits_{i\in L^-}\sqrt{\frac{2\widehat{P}(i;s,a)(1-\widehat{P}(i;s,a))\ln(2/\delta')}{N-1}}\\ =& (\sum\limits_{i'\in L^-}\widehat{P}(i';s,a))\sum\limits_{i\in L^-}\frac{\widehat{P}(i;s,a)}{\sum\limits_{i''\in L^-}\widehat{P}(i'';s,a)}\sqrt{\frac{2(1-\widehat{P}(i;s,a))\ln(2/\delta')}{\widehat{P}(i;s,a)(N-1)}}\\
	    \leq & (\sum\limits_{i'\in L^-}\widehat{P}(i';s,a))\sqrt{\sum\limits_{i\in L^+}\frac{2\ln(2/\delta')}{\sum\limits_{i''\in L^-}\widehat{P}(i'';s,a)(N-1)}} \\
	    = & \sqrt{(\sum\limits_{i'\in L^-}\widehat{P}(i';s,a))\frac{\sum\limits_{i\in L^-}2\ln(2/\delta')}{N-1}} \\
	    \leq & \sqrt{\frac{2uD\ln(2/\delta')}{N-1}}.
	\end{align}
	The first inequality holds with the Cauchy–Schwarz inequality, and the second holds due to inequality~\eqref{eq:sum_hat_p}.
	
	Now we have that
    \begin{align}
        ||\widetilde{\sigma}'(s,a)-{\sigma}'(s,a)||_1 \leq & 2\left(\sqrt{\frac{2u\ln(2/\delta')}{N-1}}+\sqrt{\frac{2uD\ln(2/\delta')}{N-1}}+\frac{7D\ln(2/\delta')}{3N}\right)\\
        \leq & 2\left(\sqrt{\frac{3u\ln(2/\delta')}{N}}+\sqrt{\frac{3uD\ln(2/\delta')}{N}}+\frac{7D\ln(2/\delta')}{3N}\right).
    \end{align}	
    
    Now if we have
    \begin{equation}
        N\geq \frac{49D\ln(2/\delta')}{27u},
    \end{equation}
    we have that 
    \begin{equation}
        \sqrt{\frac{3uD\ln(2/\delta')}{N}}\geq\frac{7D\ln(2/\delta')}{3N}.
    \end{equation}
    
    Then we have 
    \begin{equation}
        ||\widetilde{\sigma}^*(s,a)-\widehat{\sigma}^*(s,a)||_1 \leq 6\sqrt{\frac{3uD\ln(2/\delta')}{N}}.
    \end{equation}
	
	Thus we have that: 
	\begin{align}
	\label{eq:tvdc_case2_e_22}
	    e_{22}\leq 6H^2\sqrt{\frac{3uD\ln(2/\delta')}{N}}.
	\end{align}

    \textbf{Upper bounding $Err(\widehat{\pi}^*)$.}
    
    Finally if we have that
    \begin{equation}
        N \geq \max\Big(\frac{49D\ln(2/\delta')}{27u},\frac{D\ln(2/\delta')}{18u},\frac{\ln(2/\delta')}{18}\Big) = \frac{49D\ln(2/\delta')}{27u},
    \end{equation}
    we combine Eq.~\eqref{eq:tvdc_case2_error_bound}, \eqref{eq:tvdc_e_12_bound} and \eqref{eq:tvdc_e_22_bound} to get that
    \begin{align}
        Err(\widehat{\pi}^*)\leq & 2H\beta(N) + 6H^2(\sqrt{2}+\sqrt{3})\sqrt{\frac{uD\ln(2/\delta')}{N}}\\
        \leq & 4H^2\sqrt{\frac{2\ln(2/\delta')}{N}} + 24H^2\sqrt{\frac{uD\ln(2/\delta')}{N}}.
    \end{align}
    
    If we have 
    \begin{equation}
        N\geq \max\left(\frac{16H^4\ln(2/\delta')}{\epsilon^2}(\sqrt{2}+6\sqrt{uD})^2,\frac{49D\ln(2/\delta')}{27u}\right),
    \end{equation}
    then by letting $\delta'=\delta/(4SA)$ and using the union bound, with a probability at least $1-\delta$, 
    \begin{equation}
        Err(\widehat{\pi}^*)\leq \epsilon.
    \end{equation}
    
    Thus we finish the proof for Theorem \ref{thm:tvdc_thm_modified}.

    \section{Proof for other lemmas}
    
    For clarity, for those lemmas which appear in the proof of theorems, we put their proof in this section.
    
    \subsection{Proof for Lemma~\ref{lem:err_decompose}}
    \label{app:proof_err_decompose}
		We have that
		\begin{align*}
		Err(\widehat{\pi}^*) = & \min\limits_{M\in\mathcal{C}(M^s)}{V}^{\pi^*}(s_0;M) - \min\limits_{M\in\mathcal{C}(M^s)}{V}^{\widehat{\pi}^*}(s_0;M)\\
		= & \mathcal{V}^{\pi^*,\sigma^*}_{M^s}(s_0) - \mathcal{V}^{\widehat{\pi}^*,{\sigma}'}_{M^s}(s_0)\\
		= & \underbrace{\mathcal{V}^{\pi^*,\sigma^*}_{M^s}(s_0) - \mathcal{V}^{\widehat{\pi}^*,\widehat{\sigma}^*}_{\widehat{M}}(s_0)}_{\kappa_1} + \underbrace{\mathcal{V}^{\widehat{\pi}^*,\widehat{\sigma}^*}_{\widehat{M}}(s_0) - \mathcal{V}^{\widehat{\pi}^*,{\sigma}'}_{M^s}(s_0)}_{\kappa_2}.
		\end{align*}
		
		Then we upper bound $\kappa_1$ and $\kappa_2$ separately.
		
		\begin{align*}
		\kappa_1 = & \mathcal{V}^{\pi^*,\sigma^*}_{M^s}(s_0) - \mathcal{V}^{\pi^*,\widetilde{\sigma}^*}_{M^s}(s_0) + \mathcal{V}^{\pi^*,\widetilde{\sigma}^*}_{M^s}(s_0) - \mathcal{V}^{{\pi}^*,\widehat{\sigma}^*}_{\widehat{M}}(s_0) + \mathcal{V}^{{\pi}^*,\widehat{\sigma}^*}_{\widehat{M}}(s_0) -  \mathcal{V}^{\widehat{\pi}^*,\widehat{\sigma}^*}_{\widehat{M}}(s_0)\\
		\leq & \mathcal{V}^{\pi^*,\widetilde{\sigma}^*}_{M^s}(s_0) - \mathcal{V}^{{\pi}^*,\widehat{\sigma}^*}_{\widehat{M}}(s_0).
		\end{align*}
		Here the inequality holds due to the definition of $\sigma^*$ and $\widehat{\pi}^*$.
		
		\begin{align*}
		\kappa_2 = & \mathcal{V}^{\widehat{\pi}^*,\widehat{\sigma}^*}_{\widehat{M}}(s_0) - \mathcal{V}^{\widehat{\pi}^*,\widetilde{\sigma}'}_{\widehat{M}}(s_0) 
		+ \mathcal{V}^{\widehat{\pi}^*,\widetilde{\sigma}'}_{\widehat{M}}(s_0) -  \mathcal{V}^{\widehat{\pi}^*,{\sigma}'}_{M^s}(s_0)\\
		\leq & \mathcal{V}^{\widehat{\pi}^*,\widetilde{\sigma}'}_{\widehat{M}}(s_0) -  \mathcal{V}^{\widehat{\pi}^*,{\sigma}'}_{M^s}(s_0).
		\end{align*}
		Here the inequality holds due to the definition of $\widehat{\sigma}^*$.
		
		Then we finish the proof.
    
    \subsection{Proof for Lemma~\ref{lem:e12_ell}}
    \label{app:proof_e12_ell}
	    We first define two gap vectors $\widetilde{w}$ and $w$. For each $i\in[D_{sa}]$, we define 
		\begin{align}
		    \widetilde{w}(i)= & \ell^*(i)-\widetilde{\ell}(i),\\
		    w(i)= & \ell^*(i) -\ell(i).
		\end{align}
		
		We define four sets $\widetilde{W}^+$, $\widetilde{W}^-$, $W^+$ and $W^-$ as
		\begin{align}
		    \widetilde{W}^+ = & \{i\in[D_{sa}]:\widetilde{w}(i)>0\},\\
		    \widetilde{W}^- = & \{i\in [D_{sa}]:\widetilde{w}(i)<0\},\\
		    W^+ = & \{i\in[D_{sa}]:w(i)>0\},\\
		    W^- = & \{i\in [D_{sa}]:w(i)<0\}.
		\end{align}
		
		According to our definition of $\ell^*$, $\widetilde{\ell}$ and $\ell$, we have that
		\begin{align}
		    ||\ell^* -\widetilde{\ell}||_1 = & \sum\limits_{i\in\widetilde{W}^+}\widetilde{w}(i) - \sum\limits_{i\in\widetilde{W}^-}\widetilde{w}(i),\\
		    \sum\limits_{i\in\widetilde{W}^+}\widetilde{w}(i)  = & - \sum\limits_{i\in\widetilde{W}^-}\widetilde{w}(i),\\
		    ||\ell^* -\ell||_1 = & \sum\limits_{i\in W^+}w(i) - \sum\limits_{i\in W^-} w(i).
		\end{align}

    Now we consider two cases.
    
    \textbf{case 1:} $\sum\limits_{i\in W^+}w(i) \geq - \sum\limits_{i\in W^-} w(i)$.
    
    Now we aim to show that
    \begin{equation}
        \sum\limits_{i\in\widetilde{W}^+}\widetilde{w}(i) \leq \sum\limits_{i\in W^+}w(i).\label{eq:e_12_w_case1}
    \end{equation}
    If Eq.~\eqref{eq:e_12_w_case1} doesn't hold, we can build a new vector $\ell''$ such that
    \begin{equation}
        ||\ell^* - \widetilde{\ell}||_1 > ||\ell^* - \ell''||_1.
    \end{equation}
    $\ell''$ can be build by letting $\ell''(i)=\ell(i)$ if $i\in W^+\cup W^-$ and then refining $i\in [D_{sa}]/W^+$ to satisfy that $\ell''\in U(P^s)$.
    
    Thus with Eq.~\eqref{eq:e_12_w_case1}, we have that
    \begin{align}
        ||\ell^* - \widetilde{\ell}||_1 = & 2 \sum\limits_{i\in\widetilde{W}^+}\widetilde{w}(i) \\
        \leq & 2\sum\limits_{i\in W^+}w(i) \\
        \leq & 2 ||\ell^* - \ell||_1.
    \end{align}

\textbf{Case 2:} $\sum\limits_{i\in W^+}w(i) < - \sum\limits_{i\in W^-} w(i)$.
    
    Now we aim to show that
    \begin{equation}
        -\sum\limits_{i\in\widetilde{W}^-}\widetilde{w}(i) \leq -\sum\limits_{i\in W^-}w(i).\label{eq:e_12_w_case2}
    \end{equation}
    If Eq.~\eqref{eq:e_12_w_case2} doesn't hold, we can build a new vector $\ell''$ such that
    \begin{equation}
        ||\ell^* - \widetilde{\ell}||_1 > ||\ell^* - \ell''||_1.
    \end{equation}
    $\ell''$ can be build by letting $\ell''(i)=\ell(i)$ if $i\in W^+\cup W^-$ and then refining $i\in [D_{sa}]/W^-$ to satisfy that $\ell''\in U(P^s)$.
    
    Thus with Eq.~\eqref{eq:e_12_w_case2}, we have that
    \begin{align}
        ||\ell^* - \widetilde{\ell}||_1 = & -2 \sum\limits_{i\in\widetilde{W}^-}\widetilde{w}(i) \\
        \leq & -2\sum\limits_{i\in W^-}w(i) \\
        \leq & 2 ||\ell^* - \ell||_1.
    \end{align}

    Thus we finish the proof.
    
    \subsection{Proof for Lemma~\ref{lem:e22_ell}}
    \label{app:proof_e22_ell}
	    We first define two gap vectors $\widetilde{w}$ and $w$. For each $i\in[D_{sa}]$, we define 
		\begin{align}
		    \widetilde{w}(i)= & \ell'(i)-\widetilde{\ell}(i),\\
		    w(i)= & \ell^*(i) -\ell(i).
		\end{align}
		
		We define four sets $\widetilde{W}^+$, $\widetilde{W}^-$, $W^+$ and $W^-$ as
		\begin{align}
		    \widetilde{W}^+ = & \{i\in[D_{sa}]:\widetilde{w}(i)>0\},\\
		    \widetilde{W}^- = & \{i\in [D_{sa}]:\widetilde{w}(i)<0\},\\
		    W^+ = & \{i\in[D_{sa}]:w(i)>0\},\\
		    W^- = & \{i\in [D_{sa}]:w(i)<0\}.
		\end{align}
		
		According to our definition of $\ell'$, $\widetilde{\ell}$ and $\ell$, we have that
		\begin{align}
		    ||\ell' -\widetilde{\ell}||_1 = & \sum\limits_{i\in\widetilde{W}^+}\widetilde{w}(i) - \sum\limits_{i\in\widetilde{W}^-}\widetilde{w}(i),\\
		    \sum\limits_{i\in\widetilde{W}^+}\widetilde{w}(i)  = & - \sum\limits_{i\in\widetilde{W}^-}\widetilde{w}(i),\\
		    ||\ell' -\ell||_1 = & \sum\limits_{i\in W^+}w(i) - \sum\limits_{i\in W^-} w(i).
		\end{align}

    Now we consider two cases.
    
    \textbf{case 1:} $\sum\limits_{i\in W^+}w(i) \geq - \sum\limits_{i\in W^-} w(i)$.
    
    Now we aim to show that
    \begin{equation}
        \sum\limits_{i\in\widetilde{W}^+}\widetilde{w}(i) \leq \sum\limits_{i\in W^+}w(i).\label{eq:e_22_w_case1}
    \end{equation}
    If Eq.~\eqref{eq:e_22_w_case1} doesn't hold, we can build a new vector $\ell''$ such that
    \begin{equation}
        ||\ell' - \widetilde{\ell}||_1 > ||\ell' - \ell''||_1.
    \end{equation}
    $\ell''$ can be build by letting $\ell''(i)=\ell(i)$ if $i\in W^+\cup W^-$ and then refining $i\in [D_{sa}]/W^+$ to satisfy that $\ell''\in U(\widehat{P})$.
    
    Thus with Eq.~\eqref{eq:e_22_w_case1}, we have that
    \begin{align}
        ||\ell' - \widetilde{\ell}||_1 = & 2 \sum\limits_{i\in\widetilde{W}^+}\widetilde{w}(i) \\
        \leq & 2\sum\limits_{i\in W^+}w(i) \\
        \leq & 2 ||\ell' - \ell||_1.
    \end{align}

\textbf{Case 2:} $\sum\limits_{i\in W^+}w(i) < - \sum\limits_{i\in W^-} w(i)$.
    
    Now we aim to show that
    \begin{equation}
        -\sum\limits_{i\in\widetilde{W}^-}\widetilde{w}(i) \leq -\sum\limits_{i\in W^-}w(i).\label{eq:e_22_w_case2}
    \end{equation}
    If Eq.~\eqref{eq:e_22_w_case2} doesn't hold, we can build a new vector $\ell''$ such that
    \begin{equation}
        ||\ell' - \widetilde{\ell}||_1 > ||\ell' - \ell''||_1.
    \end{equation}
    $\ell''$ can be build by letting $\ell''(i)=\ell(i)$ if $i\in W^+\cup W^-$ and then refining $i\in [D_{sa}]/W^-$ to satisfy that $\ell''\in U(P^s)$.
    
    Thus with Eq.~\eqref{eq:e_22_w_case2}, we have that
    \begin{align}
        ||\ell' - \widetilde{\ell}||_1 = & -2 \sum\limits_{i\in\widetilde{W}^-}\widetilde{w}(i) \\
        \leq & -2\sum\limits_{i\in W^-}w(i) \\
        \leq & 2 ||\ell' - \ell||_1.
    \end{align}

    Thus we finish the proof.
    
    \section{Empirical results}
    
    Here we give detailed information for the simple example. 

    \begin{figure}[hbt]
        \centering
        \includegraphics[width=.7\textwidth]{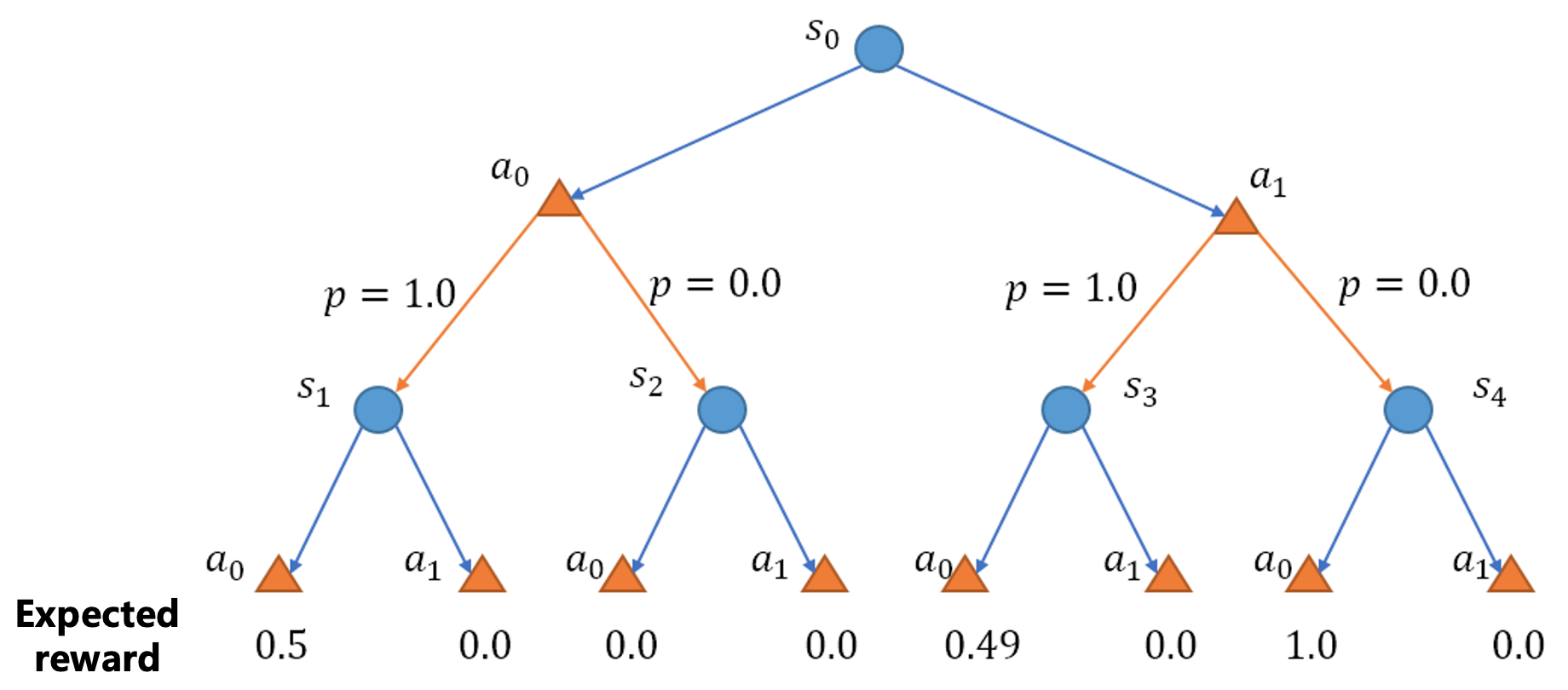}
        \caption{The structure of the MDP}
        \label{fig:mdp1}
    \end{figure}
    
    We consider a training MDP $M^s$ with only 2 layers. Depth 1 has only one state, i.e. $s_0$ and depth 2 has four states. The agent has two actions $a_0$ and $a_1$ at each state. The agent receives a reward at the end of one episode. The structure of the MDP and the rewards are given in Fig.~\ref{fig:mdp1}. It easy to see that the optimal strategy for $M^s$ is always to choose $a_0$ and the optimal reward is $0.5$.
    
    Assume that the true MDP $M^*$ and $M^s$ differs in the transitions. We assume a perturbation range $u$ for this example. That is, $P^*(s_2|s_0, a_0)\in[0,u)$  and $P^*(s_4|s_0,a_1)\in[0,u)$. Now if $u>0.02$, the robust policy for the agent is to choose $a_1$ at $s_0$. It can be seen that methods concentrating on $M^s$ can hardly identify this robustness issue since they do not pay attention to $s_4$ whose reaching probability in $M^s$ is $0$. 
    
    We aim to use this simple case to show that: (1) the optimal policy of $M^s$ may be a bad policy for $M^*$; (2) the online training process for RMDP can be very inefficient. Thus we test 4 methods: 
    \begin{itemize}
        \item OPT: solving the optimal policy of $M^s$. For each $(s,a)$, we sample $2000$ times to learn models and then use the model to find the optimal policy of $M^s$. It is easy to see that OPT can find choosing $a_0$ is the optimal policy.
        \item RPS: our robust policy solving method. We also sample each $(s,a)$ for $2000$ times and then use our algorithm for PWC to find the robust policy.
        \item RQ-learning: the robust version of Q-learning. The update step for Q learning use a minimization step, as given in \citet{roy2017reinforcement}.
        \item RSARSA: the robust version of SARSA. The update step involves the minimization step, as given in \citet{roy2017reinforcement}.
    \end{itemize}
    Each method can interact with the training environment for $20000$ times. We choose $u\in\{0.001,0.005,0.01,0.05,0.1,0.5\}$. 
    
    \begin{figure}[hbt]
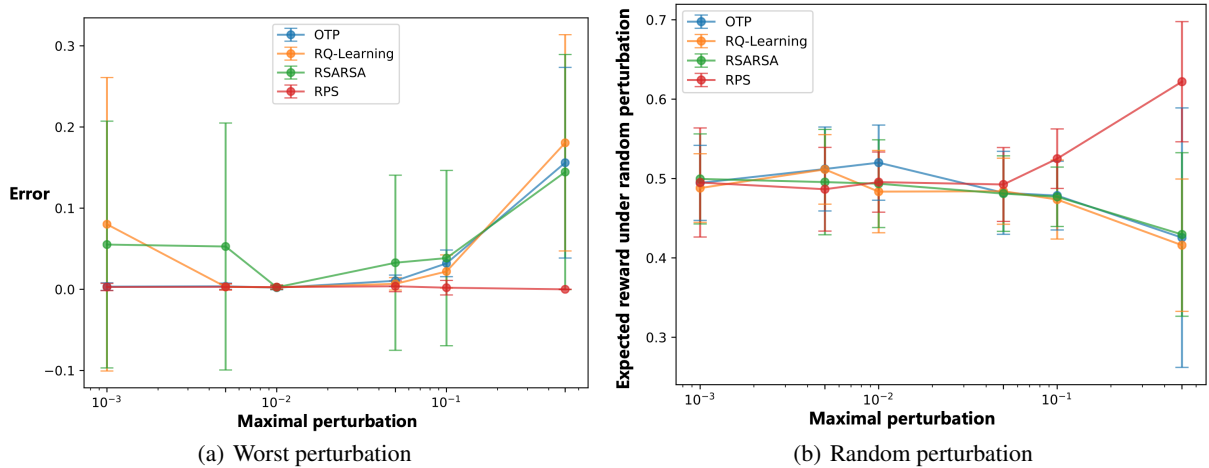

        \centering
        \subfigure[Worst perturbation]{\includegraphics[width=.45\textwidth]{pic/err.png}\label{fig:err1}}
        \subfigure[Random perturbation]{\includegraphics[width=.45\textwidth]{pic/test.png}\label{fig:rand_res1}}
        \caption{Results on 4 methods}
    \end{figure}
    
    The results under worst-case perturbation are shown in Fig.~\ref{fig:err1}. When testing the algorithms, we choose $P^*(s_2|s_0,a_0)=u$ and $P^*(s_4|s_0,a_1)=u$. It can be shown that only RPS can always finding the robust solution while other three methods suffer a large error for a large $u$. OTP suffer a large error because it ignores the effect of the perturbation. Online methods RSARSA and RQ-learning also fail because these online methods cannot reach $s_4$ at all. They cannot give a proper estimation for $s_4$ and cannot find the near-optimal robust solution.
    
    We also show the results on random perturbation in Fig.~\ref{fig:rand_res1}. We uniformly choose random values for $P^*(s_2|s_0,a_0)\in[0,u)$ and $P^*(s_4|s_0,a_1)\in[0,u)$. The random perturbation here means that the difference between the training and testing transitions is sampled uniformly from the perturbation range. It can be seen that for a relatively large perturbation, a robust policy indeed gives higher expected rewards, since it takes the risks into consideration. 

\end{document}